\documentclass{article} 
\usepackage{iclr2022_conference,times}

\usepackage{hyperref}       
\usepackage{url}            
\usepackage{booktabs}       
\usepackage{amsfonts}       
\usepackage{nicefrac}       
\usepackage{microtype}      
\usepackage{lipsum}
\usepackage{graphicx}
\usepackage{algorithmic}
\usepackage{algorithm}
\usepackage{amsmath}
\usepackage{amsthm}
\usepackage{multicol}
\usepackage{subcaption}
\usepackage{amssymb}
\usepackage{array}
\usepackage{wrapfig}
\usepackage{enumitem}

\title{ SubMix: Practical Private Prediction for Large-scale Language Models}

\author{Antonio A. Ginart\thanks{Work done while interning at Facebook AI Research.} \\
Department of Electrical Engineering \\
Stanford University \\
\texttt{tginart@stanford.edu} \\
\And
Laurens van der Maaten \\
Facebook AI Research \\
\texttt{lvdmaaten@fb.com} \\
\AND
James Zou \\
Department of Biomedical Data Science \\
Stanford University \\
\texttt{jamesz@stanford.edu} \\
\And
\hspace{18.7ex} Chuan Guo \\
\hspace{19.3ex} Facebook AI Research \\
\hspace{19.3ex} \texttt{chuanguo@fb.com} \\
}

\iclrfinalcopy 

\newtheorem{theorem}{Theorem}[section]
\newtheorem{corollary}{Corollary}[theorem]

\newtheorem{proposition}[theorem]{Proposition}
\newtheorem{remark}[theorem]{Remark} 
\theoremstyle{definition}
\newtheorem{definition}{Definition}[section]

\begin{document}
\maketitle
\begin{abstract}
\footnotesize
\vspace{-5pt}
Recent data-extraction attacks have exposed that language models can memorize some training samples verbatim. This is a vulnerability that can compromise the privacy of the model’s training data. In this work, we introduce \textsc{SubMix}: a practical protocol for private next-token prediction designed to prevent privacy violations by language models that were fine-tuned on a private corpus after pre-training on a public corpus. We show that \textsc{SubMix} limits the leakage of information that is unique to any individual user in the private corpus via a relaxation of group differentially private prediction. Importantly, \textsc{SubMix} admits a tight, data-dependent privacy accounting mechanism, which allows it to thwart existing data-extraction attacks while maintaining the utility of the language model. \textsc{SubMix} is the first protocol that maintains privacy even when publicly releasing tens of thousands of next-token predictions made by large transformer-based models such as GPT-2.
\vspace{-5pt}
\end{abstract}


\section{Introduction}

The advent of transformers~\citep{vaswani2017attention} has fostered a dramatic advancement in the capabilities of generative neural language models (LMs), enabling large-scale models such as GPT~\citep{radford2019language, brown2020language} to generate realistic, human-like text. Unfortunately, these impressive capabilities also come at a cost to privacy, as the amount of excess parameters in the LM enables it to memorize certain training samples. Consequently, several recent works have demonstrated practical training-data extraction attacks that reproduce entire sentences from the training dataset verbatim by querying the LM as an API~\citep{carlini2019secret, carlini2020extracting}. These attacks expose the privacy risks of large-scale LMs, especially when their training data contains sensitive information such as addresses and personal ID numbers.




Existing solutions to data extraction attacks focus on using differential privacy (DP;~\citet{dwork2014algorithmic}), which provably protects against privacy attacks~\citep{yeom2018privacy}. Techniques such as DP-SGD~\citep{abadi2016deep} have been applied to train differentially private neural networks on both vision and language tasks~\citep{mcmahan2017learning, papernot2018scalable}. However, the threat model in DP-SGD implicitly assumes that the adversary has full access to the private model's parameters and gradients during training, which results in pessimistic information leakage bounds that are unreasonable for most models. Indeed, existing work only performs DP-SGD training of small feedforward networks~\citep{kerrigan2020differentially} and RNNs~\citep{mcmahan2017learning, ramaswamy2020training}, often with an unsatisfactory privacy-utility trade-off. Training large machine-learning models with DP-SGD remains an open challenge~\citep{jayaraman2019evaluating, tramer2020differentially}.


 
Our study deviates from prior work by, instead, considering the problem of \emph{private prediction}~\citep{dwork2018privacy} using non-private language models fine-tuned on a private corpus. 
We propose \textsc{SubMix}, a novel private prediction mechanism for answering next-token queries. Focusing on private prediction affords \textsc{SubMix} three notable advantages: (1) Private prediction does not require modification of the training algorithm, which makes use of large-scale LMs feasible. (2) Private prediction allows us to leverage the probabilistic nature of next-token sampling for highly efficient privacy accounting. (3) Private prediction allows us to leverage public pre-trained LMs\footnote{ We do not provide privacy guarantees or text extraction protection for the public corpus on which the LMs are pre-trained; our privacy guarantees only on apply to the private corpus on which the LM is fine-tuned.} to obtain private predictive distributions that do not require noise addition to privatize the model's predictions.

\textsc{SubMix} utilizes an ensemble of LMs fine-tuned on disjoint parts of the private corpus and privatizes predictions by mixing the next-token distribution with that of a public pre-trained LM. The mixing weight is adaptively tuned based on the degree of consensus among models in the ensemble. If all models predict the same next-token distribution, then it is impossible for the next token to leak sensitive information about any unique individual so no mixing is required. By contrast, if models in the ensemble have high disagreement, \textsc{SubMix} will mix predictions with those of the public pre-trained model to minimize privacy leakage. This allows \textsc{SubMix} to perform accurate next-token prediction for most queries while preserving the privacy of the private corpus.

For any sequence of next-token queries issued to \textsc{SubMix}, we measure the amount of privacy leakage in the response using R\'{e}nyi divergence~\citep{renyi1961measures}.
Our privacy notion, which we refer to as \emph{operational privacy}, is a sufficient condition for preventing samples that are unique to any user from being generated by the \textsc{SubMix} mechanism. Importantly, operational privacy allows us to perform tight data-dependent privacy accounting to upper bound the privacy loss of \textsc{SubMix} when answering a variable-length query sequence. Concretely, when answering up to $1,024$ next-token queries, \textsc{SubMix} realizes nearly $75\%$ of the perplexity improvement that non-private fine-tuning would have achieved on GPT-2 models \citep{radford2019language}, with privacy leakage as small as $\epsilon=2$. We also show that \textsc{SubMix} can effectively prevent existing data extraction attacks against GPT-2. 

\section{Problem Formulation}

We begin by setting up the problem of private next-token prediction and reviewing existing literature on differential privacy. We then define and discuss the notion of operational privacy for \textsc{SubMix}.

\subsection{Preliminaries}

Let $\Sigma$ denote a fixed finite vocabulary set. We use lower-case letters to denote single tokens (such as $x \in \Sigma$) and use bold font to denote contexts or strings of tokens (such as $\mathbf{x} \in \Sigma^*$). A (causal) language model $h$ is a mapping from \emph{context strings} to a distribution over next tokens:  $h: \Sigma^* \rightarrow \mathbf{\Delta}^{|\Sigma|}$, where $\mathbf{\Delta}^{|\Sigma|}$ is the $|\Sigma|$-dimensional probability simplex. For a particular context $\mathbf{x} \in \Sigma^*$, let $h(\mathbf{x})$ denote the next-token distribution vector obtained from evaluating $h$ on context $\mathbf{x}$, and let $h(z|\mathbf{x}) \in [0,1]$ denote the probability mass on a token $z \in \Sigma$.

\paragraph{User-level Corpus} Let $\mathcal{D}$ denote a dataset of unstructured text, which is a set of token sequences $\mathbf{x} \in \Sigma^*$. We assume that $\mathcal{D}$ is generated by a set of $n$ distinct users, each holding a subset $\mathcal{D}_i$ of the full dataset $\mathcal{D}$, \emph{i.e.}, $\mathcal{D} = \bigcup_{i=1}^n \mathcal{D}_i$ with $\mathcal{D}_i \cap \mathcal{D}_j = \emptyset$ for $i \neq j$. We refer to each $\mathcal{D}_i$ as a \emph{user-level corpus} for user $i$. As a concrete example, in the context of social media posts, $\mathcal{D}_i$ would contain \emph{all} of the non-public posts made by user $i$. We aim to provide privacy guarantees for a model that is non-privately fine-tuned on the dataset $\mathcal{D}$.


\paragraph{Next-token Prediction} One popular use case for language models is to perform \emph{next-token prediction}, that is, return a token $z$ when queried with a context $\mathbf{x}$. Such a query-answering API is useful for applications such as smart keyboard for auto-correction and text completion~\citep{mirowski2015dependency,hertel2019thesis}. Typical approaches for next-token prediction involve sampling $z$ from the next-token distribution vector $h(\mathbf{x})$; see \citet{holtzman2019curious}. Large transformers trained on unstructured internet text have achieved remarkable success for this task, producing natural-looking sentences via sequentially generating next tokens from a given prompt~\citep{brown2020language}.

\paragraph{Text Extraction Attacks} Recent studies have shown that it is possible for next-token prediction APIs to reveal sensitive private information contained in the training dataset. \cite{carlini2019secret} defined $\kappa$-\emph{eidetic memorization} to formalize the notion that extraction of strings that are uncommon in the corpus can lead to violations of user privacy.

\begin{definition}[$\kappa$-eidetic memorization \citep{carlini2019secret}]
A string $s$ is $\kappa$-eidetic memorized by an LM $h$ if $s$ is \emph{extractable}\footnote{\cite{carlini2019secret} define \emph{text extraction} informally. In the supplement, we formalize it within the framework of statistical hypothesis testing and show that differential privacy is sufficient to prevent eidetic memorization.} from $h$ and $s$ appears in at most $\kappa$ examples in the training
data $\mathcal{D}$.
\label{def:eidetic_mem}
\end{definition}

\cite{carlini2019secret} showed that if the training dataset contains token sequences of the form: ``\texttt{My social security number is $\square \square \square$-$\square \square$-$\square \square \square \square$}'' where $\square$ represents a digit of a user's social security number (SSN), then it is subtantially more likely for the LM trained on $\mathcal{D}$ to generate the exact SSN appearing in $\mathcal{D}$ compared to a random SSN. As a result, it is possible to design an efficient \emph{extraction attack} that reproduces such unique sequences in the training dataset. \cite{carlini2020extracting} further extended this attack to large transformer-based LMs such as GPT2~\citep{radford2019language}, extracting memorized personal information such as name and address contained in the model's training dataset. Motivated by these shortcomings, this paper studies notions of privacy that can prevent such text-extraction attacks while preserving the model's utility. 

\subsection{Differential Privacy}

Differential privacy~\citep{dwork2014algorithmic} is a powerful mathematical framework for privacy-preserving data analysis. The underlying principle in differential privacy and all its variants is the notion of \emph{indistinguishability}. Informally, a mechanism $\mathcal{M}$ is private if, given two adjacent datasets $\mathcal{D}$ and $\mathcal{D}'$, the mechanism's outputs $\mathcal{M}(\mathcal{D})$ and $\mathcal{M}(\mathcal{D}')$ are approximately indistinguishable. Hence by observing the output of $\mathcal{M}$, it is difficult for an adversary to discern the difference between $\mathcal{D}$ and $\mathcal{D}'$. The above informal definition of privacy can be made mathematically precise by specifying: (1) the notion of adjacency between datasets $\mathcal{D}$ and $\mathcal{D}'$, and (2) the notion of approximate indistinguishability.

\paragraph{Differentially Private Training} Prior work on private LM training~\citep{mcmahan2017learning, ramaswamy2020training} adopted the definition of \emph{user-level adjacency}: $\mathcal{D}$ and $\mathcal{D}'$ are adjacent if they differ in a single user's data. Approximate indistinguishability is defined in terms of divergences and is applied to the trained model: The private training algorithm $\mathcal{M}(\mathcal{D})$ induces a distribution over models, and indistinguishability requires that $D(\mathcal{M}(\mathcal{D}) || \mathcal{M}(\mathcal{D}')) < \epsilon$ for some divergence $D$ and small constant $\epsilon>0$. Popular choices include the \emph{max divergence}~\citep{dwork2014algorithmic} and the \emph{R\'{e}nyi divergence of order $\alpha$}~\citep{renyi1961measures}:
$$D_\infty(P || Q) = \sup_{x \in \mathrm{supp}(Q)} \log P(x) - \log Q(x), \text{\quad\quad} D_\alpha(P || Q) = \frac{1}{\alpha-1} \log \mathbb{E}_{x \sim Q} \left[ P(x) / Q(x) \right]^\alpha.$$
Specializing to the choice of R\'{e}nyi divergence, we define \emph{user-level R\'{e}nyi differential privacy} (RDP; \citet{mironov2017renyi}) for private training as follows.


\begin{definition}[User-level RDP for private training]
For $\alpha > 1$, let $D_{\alpha}$ denote the order-$\alpha$ R\'{e}nyi divergence. A private training algorithm $\mathcal{M}$ is an $(\alpha, \epsilon)$-RDP mechanism if for any $\mathcal{D}$ and $\mathcal{D'}$ that differ in only one user's data $\mathcal{D}_i$, we have $D_{\alpha}(\mathcal{M}(\mathcal{D}) ||\mathcal{M}(\mathcal{D}') ) \leq \epsilon$.
\label{def:private_training}
\end{definition}

In order to satisfy the criteria in Definition \ref{def:private_training} for neural language models, the standard approach is to use DP-SGD~\citep{abadi2016deep} to inject noise into the gradients computed at every iteration of SGD training, and use composition theorems to bound the total privacy leakage across iterations.

\paragraph{Differentially Private Prediction} Private prediction differs from private training in that the notion of approximate indistinguishability applies to a sequence of predictions made by a \emph{private prediction protocol} $\mathcal{P}$, rather than to a privately trained model. Formally, at each time step $t$, an adversary $\mathsf{Adv}$ (potentially adaptively) issues a context string $\mathbf{x}_t$, and the private prediction protocol $\mathcal{P}$ responds by generating a next token $y_t \in \Sigma$. We let $\mathcal{P} \leftrightharpoons_{T} \mathsf{Adv}$ denote the sequence of query-response pairs between $\mathsf{P}$ and $\mathsf{Adv}$ up until time $T$:  $\mathcal{P} \leftrightharpoons_{T} \mathsf{Adv} = \{\mathbf{x}_t, y_t \}_{t=1}^T$. 
For a query sequence of length $T$, approximate indistinguishability requires that for adjacent datasets $\mathcal{D}, \mathcal{D}'$: \begin{equation}
D\left(\mathcal{P}(\mathcal{D}) \underset{\scriptscriptstyle{{T}}}{\leftrightharpoons}\mathsf{Adv}~||~ \mathcal{P}(\mathcal{D}') \underset{\scriptscriptstyle{{T}}}{\leftrightharpoons}\mathsf{Adv}\right) \leq \epsilon,
\label{eq:private_prediction}
\end{equation}
for some divergence $D$ and $\epsilon > 0$. We summarize the above discussion in the following R\'{e}nyi-DP variant of the definition for private prediction by \citet{dwork2018privacy}.

\begin{definition}[User-level RDP for private prediction]
Let $\alpha > 1$, $\epsilon > 0$, and $T \in \mathbb{Z}_+$. A prediction protocol $\mathcal{P}$ is $(\alpha, \epsilon, T)$-RDP if for any adversary $\mathsf{Adv}$ and any $\mathcal{D}$ and $\mathcal{D'}$ that differ in only one user's data $\mathcal{D}_i$, we have that \autoref{eq:private_prediction} holds.
\label{def:private_prediction}
\end{definition}

It is well-known that differentially private models can be used for private prediction via the \emph{post-processing theorem}~\citep{mironov2017renyi}: If $h \leftarrow \mathcal{M}(\mathcal{D})$ is a model obtained from an $(\alpha, \epsilon)$-RDP training mechanism $\mathcal{M}$, then $\mathcal{M}'(\mathbf{x}; \mathcal{D}) = h(\mathbf{x})$ is an $(\alpha,\epsilon, \infty)$-RDP private prediction mechanism for any sequence of queries (regardless of length). However, DP-SGD~\citep{abadi2016deep}---the primary mechanism for training private neural networks---makes an implicit assumption that the adversary also observes additional information that is not accessible if $h$ is used as a prediction API, and in practice, it causes the accounting mechanism in DP-SGD to vastly overestimate the privacy leakage parameter $\epsilon$~\citep{nasr2021adversary}.
One alternative is the general-purpose \emph{subsample-and-aggregate} mechanism, which adds noise to an ensemble's output in order to privatize it. This results in a trade-off between the information leakage, $\epsilon$, and the number of queries that can be answered, $T$ \citep{van2020trade}. For smaller $T$, the mechanism needs less noise to achieve a particular $\epsilon$. Conceptually, this is a step in the right direction, but the added noise greatly reduces utility and is superfluous if we can leverage pre-trained public LMs to privatize the predictive distribution.

\subsection{Operational Privacy for Private Prediction}
To remedy the problems in user-level differentially private training and prediction, we propose a different notion of privacy that is sufficient for preventing text extraction attacks, but admits more specialized privacy mechanisms with tighter privacy accounting.

Let $\mathbb{P}(\mathcal{D})$ denote the power set of $\mathcal{D}$. A \emph{partition} $\Pi \in \mathbb{P}(\mathcal{D})$ of $\mathcal{D}$ is a collection of sets $\pi$ that satisfies $\bigcup_{\pi\in \Pi} \pi = \mathcal{D}$ and that satisfies $\pi \cap \pi' = \emptyset$ for distinct $\pi,\pi' \in \Pi $. We refer to the elements of $\Pi$ as \emph{parts}. For some fixed ordering, we let $\Pi_i$ denote the $i$-th part. As a minor abuse of notation, we let $\mathcal{D} \setminus \pi$ denote the usual element-wise subtraction and write $\Pi \setminus \pi$ instead of $\Pi \setminus \{\pi\}$ for brevity. Recall the notion of the private prediction protocol $\mathcal{P}$. We augment the protocol $\mathcal{P}$ with the capability to terminate the query-response sequence at any time. With a slight abuse of notation, we denote by $T(\mathcal{P})$ the sequence length produced by $\mathcal{P}$. We define \emph{R\'{e}nyi operational privacy} (ROP) as follows.

\begin{definition}[R\'{e}nyi operational privacy for private prediction]

Let $\mathcal{D}$ be a dataset of user-level corpora and let $\Pi$ be a partition so that each user-level corpus $\mathcal{D}_i$ is contained in some part $\Pi_j \in \Pi$. For $\alpha > 1$ and $\epsilon > 0$, a prediction protocol $\mathcal{P}$ is $(\alpha, \epsilon)$-ROP for partition $\Pi$ of dataset $\mathcal{D}$ if for any part $\Pi_i \in \Pi$ and adversary $\mathsf{Adv}$, we have: $$D_{\alpha}^{\text{sym}}\left(\mathcal{P}(\Pi) \underset{\scriptscriptstyle{{T(\mathcal{P})}}}{\leftrightharpoons} \mathsf{Adv}~||~\mathcal{P}(\Pi \setminus \Pi_i) \underset{\scriptscriptstyle{{T(\mathcal{P})}}}{\leftrightharpoons} \mathsf{Adv}\right) \leq \epsilon.$$
\label{def:operational_privacy}
\end{definition}
\vspace{-15pt}
Where $D_{\alpha}^{\text{sym}}(P||Q) = \max\{D_{\alpha}(P||Q),  D_{\alpha}(Q||P)\}$. ROP differs from user-level RDP in Definition \ref{def:private_prediction} in two aspects:
\begin{enumerate}[leftmargin=*,nosep]
\item We substitute user-level adjacency with partition-level adjacency. By definition, the partition is constructed so that any user's data belongs to a single part. Readers familiar with \emph{group privacy}~\citep{dwork2014algorithmic} may recognize partition-level adjacency as a formal relaxation of the group-level adjacency used in R\'{e}nyi group differential privacy. Partition-level RDP is neither strictly weaker nor stronger than user-level RDP, and, at a cost, conversion is possible (\autoref{sec:supp_conversion}). 
\item We allow \emph{variable-length} query-response sequences by enabling the private-prediction protocol $\mathcal{P}$ to terminate\footnote{After termination, the mechanism can technically continue to issue responses, but only in a way that is entirely independent of the private corpus, \emph{e.g.}, by using a public pre-trained LM.} at will. ROP accounts for the privacy leakage in the responses made throughout the prediction protocol's operation lifetime, which is why we refer to it as \emph{operational}. By allowing for a variable-length sequence, we provide the mechanism with additional flexibility without increasing susceptibility to text extraction. Non-sensitive queries often can be answered without much privacy leakage, whereas sensitive queries may quickly exhaust the privacy budget, causing the protocol to terminate early. The protocol's decision to terminate at time $T(\mathcal{P})$ may leak some information about how sensitive the queried contexts are. However, this leakage is relatively insignificant and can be upper bounded (allowing us to convert variable-length $\epsilon$ into fixed-length; see \autoref{sec:supp_conversion}).
\end{enumerate}

\section{\textsc{SubMix}}

\begin{figure}[t]
     \centering
     \begin{subfigure}[b]{0.45\textwidth}
         \centering
         \includegraphics[width=\textwidth]{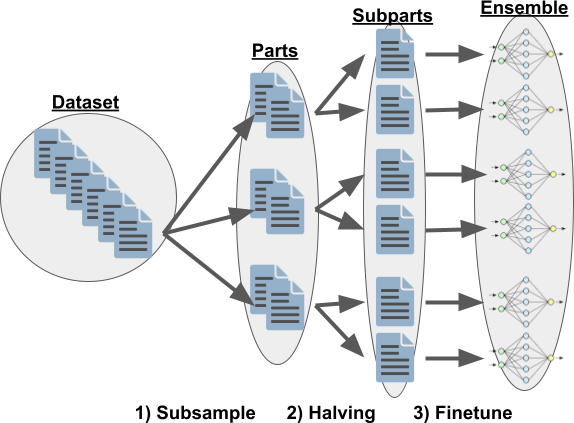}
         \caption{\small \textbf{Training.} The corpus is a dataset comprised of private user text. Each document represents all of the text corresponding to a particular user. At training time, \textsc{SubMix} learns an ensemble by: (1) subsampling the dataset into non-overlapping parts, (2) halving each part into two subparts, and (3) fine-tuning the LM on each subpart using $\mathcal{L}$.}
         \label{fig:submix_training}
     \end{subfigure}
     \hfill
     \begin{subfigure}[b]{0.49\textwidth}
         \centering
         \includegraphics[width=\textwidth]{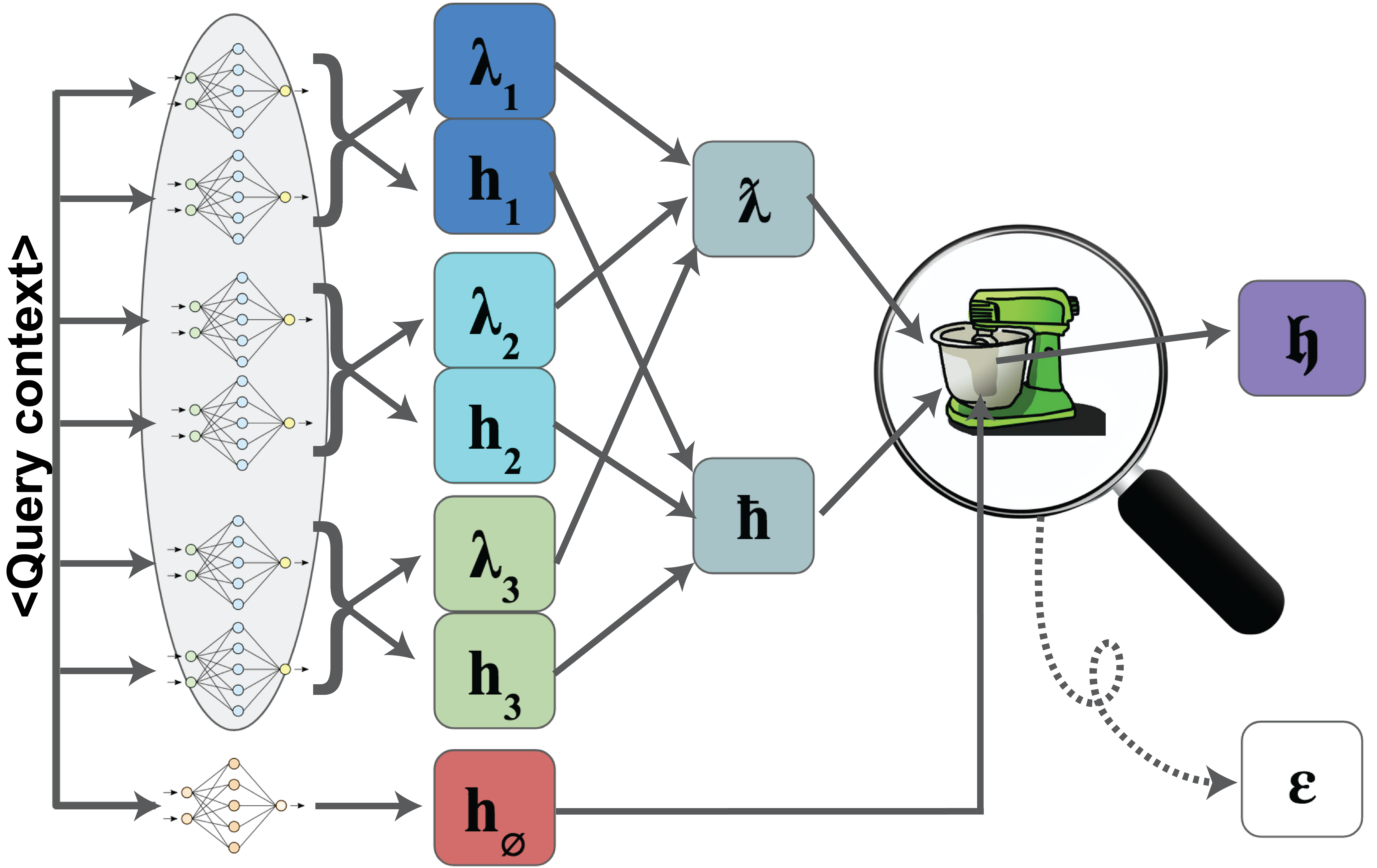}
         \caption{\small \textbf{Prediction.} The bottom-most network in the figure represents the pre-trained public model; the other networks form the ensemble of model pairs obtained after \textsc{SubMix} training. \textsc{SubMix} prediction produces a mixing weight for each model. These weights are aggregated and used to mix the ensemble predictions with the predictions of the public model.}
         \label{fig:submix_prediction}
     \end{subfigure}
\caption{Overview of \textsc{SubMix}'s training protocol (\textbf{left}) and prediction protocol (\textbf{right}).
\label{fig:submix_schematics}}
\vspace{-10pt}
\end{figure}

We introduce \textsc{SubMix}, a private next-token prediction protocol that satisfies the operational privacy definition introduced above ( \autoref{fig:submix_schematics}). \textsc{SubMix} follows the design of the subsample-and-aggregate mechanism by first forming a random partition $\Pi$ of the training dataset, with each user's data belonging to a single random part $\Pi_i$. For each part $\Pi_i$, the protocol further splits $\Pi_i$ into two \emph{subparts} $\pi_i$ and $\pi_i'$ by randomly assigning users in $\Pi_i$ to the two halves. 
\begin{wrapfigure}{r}{0.57\textwidth}
 \begin{minipage}{.99\linewidth}
\begin{algorithm}[H]
\small
\textbf{Inputs:} User-level private corpus $\mathcal{D}$, LM fine-tuning routine $\mathcal{L}$

\textbf{Outputs:} Fine-tuned LMs $h_{\pi_i}, h_{\pi_i'}$ for $i=1,\ldots,k$

\textbf{Hyperparameters:} \# of parts $k$



\begin{algorithmic}[1]
\footnotesize
\STATE $\Pi \gets $ Random $k$-fold partition of $\mathcal{D}$ with $|\Pi_{i}|\!=\!\nicefrac{|\mathcal{D}|}{k}$
\FOR{ $i \in \{1,...,k\}$}
\STATE $(\pi_i, \pi_i') \gets $ Randomly split part $\Pi_i$ into two subparts.
\STATE $h_{\pi_i} \gets \mathcal{L}(\pi_i)$, $h_{\pi'_i} \gets \mathcal{L}(\pi_i')$
\ENDFOR

\end{algorithmic}
\caption{\textsc{SubMix} Training.}\label{alg:SubMix_train}
\end{algorithm}
 \end{minipage}
\vspace{-20pt}
\end{wrapfigure}

\vspace{-5pt}
\paragraph{Fine-tuning a Pre-trained Model} Language models are often first trained on vast internet crawls to develop a general understanding of human language, and then fine-tuned on a more domain-specific dataset for the target task \citep{dai2015semi, howard2018universal}. We treat language model training as a black-box operation, and denote the training routine $\mathcal{L}(\cdot)$ as a function that returns a model $h_{\mathcal{D}} = \mathcal{L}(\mathcal{D})$. We assume access to a LM pre-trained on public data, and use routine $\mathcal{L}$ to fine-tune the LM on the private user-level corpora. Specifically, Algorithm \ref{alg:SubMix_train} fine-tunes a public pre-trained LM on each subpart $\pi_i$ and $\pi_i'$ to produce LMs $h_{\pi_i} = \mathcal{L}(\pi_i)$ and $h_{\pi_i'} = \mathcal{L}(\pi_i')$. By convention, fine-tuning on the empty set returns the public pre-trained model: $h_{\emptyset} = \mathcal{L}(\emptyset)$.

\paragraph{Next-token Distribution} Given a query context $\mathbf{x}_t$, each part $\Pi_i$ is responsible for producing a next-token probability mass function (pmf) by combining $h_{\pi_i}(\mathbf{x}_t)$ and $h_{\pi_i'}(\mathbf{x}_t)$ into $\bar{h}_i(\mathbf{x}_t) = \nicefrac{(h_{\pi_i}(\mathbf{x}_t) + h_{\pi_i'}(\mathbf{x}_t))}{2}$. \textsc{SubMix} mixes this pmf with the public pre-trained model $h_\emptyset$ to add noise to the prediction that hides private information. It does so by computing: \begin{equation*}
    h_i(\mathbf{x}_t) = \lambda^* \bar{h}_i(\mathbf{x}_t) + (1 - \lambda^*) h_\emptyset(\mathbf{x}_t),
\end{equation*}
for a suitable choice of the mixing weight $\lambda^*$. A value of $\lambda^*=0$ means the fine-tuned LMs $h_{\pi_i}$ and $h_{\pi_i'}$ are not used (no privacy loss), and $\lambda^*=1$ means no noise was added (no utility loss). We select $\lambda^*$ based on how much information about the part $\Pi_i$ is contained in the pmfs $h_{\pi_i}(\mathbf{x}_t)$ and $h_{\pi_i'}(\mathbf{x}_t)$.

Intuitively, since both $\pi_i$ and $\pi_i'$ are random samples from the same data distribution, if the models $h_{\pi_i}$ and $h_{\pi_i'}$ did not memorize the query context $\mathbf{x}_t$ then $h_{\pi_i}(\mathbf{x}_t)$ and $h_{\pi_i'}(\mathbf{x}_t)$ will be similar. Hence, the selected value of $\lambda^*$ should be close to $1$. If either $h_{\pi_i}$ or $h_{\pi_i'}$ memorized the context $\mathbf{x}_t$, then $h_{\pi_i}(\mathbf{x}_t)$ and $h_{\pi_i'}(\mathbf{x}_t)$ are dissimilar as $\pi_i$ and $\pi_i'$ have no users in common. This suggests that mixing with the pre-trained LM $h_\emptyset$ is necessary for hiding the sensitive information in $\Pi_i$, so $\lambda^*$ should be close to $0$. \textsc{SubMix} balances between these two extremes by computing a separate $\lambda_i$ for each part $\Pi_i$. Specifically, it sets a target privacy leakage $\beta > 0$ and optimizes:
\begin{equation}
    \lambda_i \leftarrow \max_{\lambda \in [0,1]} \{\lambda :\mathsf{D}_i(\mathbf{x}_t, \lambda) \leq \beta \},
    \label{eq:opt_lambda}
\end{equation}
where $\mathsf{D}_i(\mathbf{x}_t, \lambda) =D_{\alpha}\left(\lambda h_{\pi_i}(\mathbf{x}_t) + (1-\lambda)h_{\emptyset}(\mathbf{x}_t)~||~ \lambda h_{\pi_i'}(\mathbf{x}_t) + (1-\lambda)h_{\emptyset}(\mathbf{x}_t)\right)$. The final value of $\lambda^*$ is obtained by averaging the $\lambda_i$ values for $i=1,\ldots,k$, where $k$ is the number of parts.

\paragraph{Prediction and Privacy Accounting} Given the next-token pmfs $h_i(\mathbf{x}_t)$ for $i=1,\ldots,k$, \textsc{SubMix} computes the ensemble pmf, $h(\mathbf{x}_t) = \nicefrac{1}{k} \sum_{i=1}^k h_i(\mathbf{x}_t)$, and samples from it to obtain a next-token prediction. Our mechanism for selecting the mixing weight $\lambda^*$ can be shown to limit the privacy loss of a sample from $h(\mathbf{x}_t)$ under the operational privacy notion: Since each $\lambda_i$ is determined entirely by the part $\Pi_i$, the next-token pmf after removal of $\Pi_i$ can be derived in closed form. This allows us to compute the R\'{e}nyi divergence in $h(\mathbf{x}_t)$ for adjacent datasets $\Pi \setminus \Pi_i$. We present the \textsc{SubMix} prediction protocol in Algorithm \ref{alg:SubMix_pred}, and give its formal privacy analysis in the following proposition.

\begin{algorithm}[t]
\footnotesize

\textbf{Inputs:} Fine-tuned LMs $h_{\pi_i}, h_{\pi_i'}$ for $i=1,\ldots,k$, privacy parameters $\epsilon$, time step $t$, query context $\mathbf{x}_t \in \Sigma^*$

\textbf{Outputs:} Next token response $y_t \in \Sigma$

\textbf{Hyperparameters:} R\'{e}nyi divergence order $\alpha$, target leakage $\beta$

\begin{multicols}{2}
\begin{algorithmic}[1]

\IF{$t = 1$}
\STATE $\varepsilon_i \gets \epsilon$ for $i = 1,\ldots,k$
\ELSIF{$\textsf{STOP}$ has been issued}
\STATE \textbf{return } $y_t \sim h_\emptyset(\mathbf{x}_t)$
\ENDIF
\FOR{ $i = 1,\ldots,k$}
\STATE  $\bar{h}_i(\mathbf{x}_t) \gets \frac{1}{2}(h_{\pi_i}(\mathbf{x}_t) + h_{\pi'_i}(\mathbf{x}_t))$
\STATE Compute $\lambda_i$ using \autoref{eq:opt_lambda}.
\ENDFOR
\STATE $\lambda^* \gets \frac{1}{k} \sum_{i=1}^k \lambda_i$

\STATE $\bar{h}(\mathbf{x}_t) \gets \frac{1}{k} \sum_{i=1}^k \bar{h}_{i}(\mathbf{x}_t) $
\STATE $h(\mathbf{x}_t) \gets \lambda^* \bar{h}(\mathbf{x}_t) + (1-\lambda^*) h_\emptyset(\mathbf{x}_t) $

\FOR{ $i = 1,\ldots,k$}
\STATE $\lambda_{-i}^* \gets \frac{1}{k-1} \sum_{j \neq i} \lambda_j$
\STATE $\bar{h}_{-i} \gets \frac{1}{k-1} \sum_{j \neq i} \bar{h}_{j}(\mathbf{x}_t) $
\STATE $\mathfrak{h} \gets \lambda^* \bar{h}(\mathbf{x}_t) + (1-\lambda^*)h_{\emptyset}(\mathbf{x}_t)$
\STATE $\mathfrak{h}' \gets \lambda_{-i}^*\bar{h}_{-i}(\mathbf{x}_t) + (1-\lambda_{-i}^*)h_{\emptyset}(\mathbf{x}_t)\big)$
\STATE $\varepsilon_i \gets\varepsilon_i - \max\{D_\alpha(\mathfrak{h}||\mathfrak{h}'), D_\alpha(\mathfrak{h}'||\mathfrak{h}) \}$
\ENDFOR

\IF{$ \forall i$: $\varepsilon_i > 0$}
\STATE $y_t \sim h(\mathbf{x}_t)$
\ELSE
\STATE Issue \textsf{STOP} signal.
\STATE $y_t \sim h_\emptyset(\mathbf{x}_t)$
\ENDIF
\STATE \textbf{return } $y_t$

\end{algorithmic}
\end{multicols}
\caption{\textsc{SubMix} Prediction.}\label{alg:SubMix_pred}
\end{algorithm}

\begin{proposition}
\textsc{SubMix} is an $(\alpha, \epsilon)$-ROP prediction mechanism.
\end{proposition}
\vspace{-7pt}
\begin{proof} We will use the adaptive sequential composition theorem for RDP filters~\cite[Theorem 4.3]{feldman2021individual}. 
Lines 20-25 ensure that for all parts $i=1,\ldots,k$, at stopping time $T(\mathcal{P})$, the sequence of query responses $y_1,\ldots,y_{T(\mathcal{P})-1}$ satisfies: $$\sum_{t=1}^{T(\mathcal{P})-1} D_{\alpha}^\text{sym}\left(y_t \sim \mathsf{P}(\Pi)~||~y_t \sim \mathcal{P}(\Pi \setminus \Pi_i)\right) \leq \sum_{t=1}^{T(\mathcal{P})-1} \varepsilon_i(t) - \varepsilon_i(t+1),$$ where $\varepsilon_i(t)$ is the remaining privacy budget at the start of time $t$. Then by the RDP filter:
\vspace{-7pt}
\begin{align*}
 \max_i D_\alpha^\text{sym}\left(\mathsf{P}(\Pi) \leftrightharpoons \mathsf{Adv}~||~\mathsf{P}(\Pi \setminus \Pi_i) \leftrightharpoons \mathsf{Adv}\right) &\leq \max_{i} \sum_{t=1}^{T(\mathcal{P})-1} D_{\alpha}^\text{sym}\left(y_t \sim \mathsf{P}(\Pi)~||~y_t \sim \mathsf{P}(\Pi \setminus \Pi_i)\right)\\
 &\leq \max_{i} \sum_{t=1}^{T(\mathcal{P})-1} \varepsilon_i(t) - \varepsilon_i(t+1) \leq \max_i \epsilon = \epsilon.
\end{align*}
\vspace{-10pt}

To conclude the analysis, note that after \textsc{SubMix} issues the \textsf{STOP} signal, any subsequent queries are answered by $h_\emptyset$. Since $h_\emptyset$ is not a function of $\Pi$, this does not leak any additional information.
\end{proof}





\section{Experiments}

\paragraph{Datasets} We evaluate \textsc{SubMix} by fine-tuning the pre-trained GPT-2 model from HuggingFace~\citep{wolf2019huggingface} on two ``private'' datasets: (1) \texttt{Wikitext-103}~\citep{merity2016pointer}, a collection of 103 million tokens scraped from Wikipedia; and (2) \texttt{BigPatent-G} \citep{sharma2019bigpatent}, a collection of over 200,000 patents. We split the \texttt{wikitext-103} corpus into blocks of length 512 tokens and define each block as a (synthetic) user $\mathcal{D}_i$. We split \texttt{BigPatent-G} by patent and define each user to be a single patent. This setup mimics settings in which users have distinct data distributions within the text corpus.

\paragraph{Fine-Tuning \& Evaluation} We use standard hyperparameters for fine-tuning; see \autoref{sec:supp_details} for details. To assess the quality of LM predictions, we measure predictive perplexity:
$$\mathbf{PP}_{h} = \mathbb{E}_{\mathbf{x} = x_1 \cdots x_L \sim \mathcal{D}_\mathrm{heldout}} \left[ \exp \left( -\frac{1}{L} \sum_{i=1}^L \log h(x_i|x_1 \cdots x_{i-1}) \right) \right],$$ where $\mathcal{D}_\mathrm{heldout}$ is the private held-out set and $h(\cdot | x_1 \cdots x_{i-1})$ denotes the pmf for the next token given context $x_1 \cdots x_{i-1}$. In practice, we truncate the context window to a fixed maximum length $L=512$. Since each held-out sample consists of a block of $L=512$ tokens, computing $\mathbf{PP}_h$ for a single sample requires $L$ queries in total. Hence, the total number of queries $B$ is a multiple of 512.

Following \cite{geumlek2017r}, we report privacy loss in terms of $\alpha$-R\'{e}nyi divergence. Note that we can convert from $(\alpha,\epsilon)$-RDP to $(\epsilon',\delta)$-DP via $\epsilon' = \epsilon + \frac{\log(1/\delta)}{\alpha-1}$~\citep{mironov2017renyi}. In the paper, we measure $\epsilon$ using $\alpha =2$ R\'{e}nyi divergence; we present results for other values of $\alpha$ in \autoref{sec:supp_figures}.

\subsection{Privacy-Utility Trade-off}


\begin{wrapfigure}{r}{0.6\textwidth}
  \vspace{-20pt}
  \begin{center}

  \includegraphics[width=\linewidth]{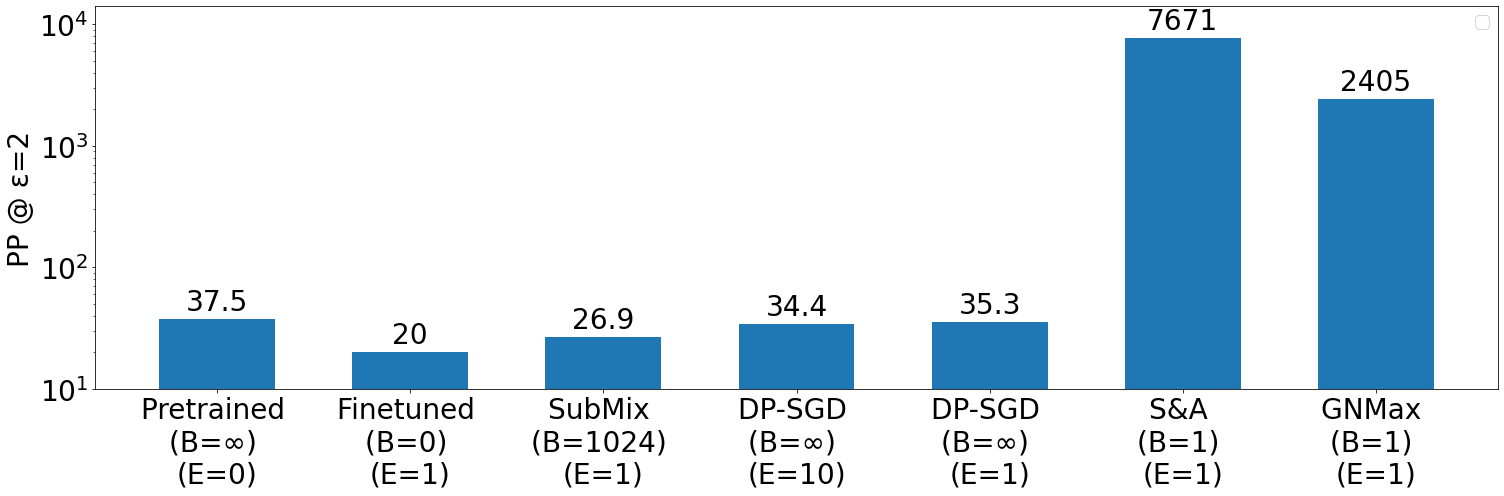} 
  \vspace{-15pt}
\caption{\small Perplexity for $\epsilon = 2$, $\alpha =2$ of four privacy-preserving mechanisms (\textsc{SubMix}, DP-SGD, S\&A, and GNMax) using GPT-2 on $\texttt{Wikitext-103}$ with varying query budgets $B$. The number of epochs of training is denoted with $E$. For DP-SGD, we report results for both one epoch ($E=1$) and ten epochs ($E=10$). Other than DP-SGD ($E=10$), all other fine-tuning methods only require one pass of training. Non-privately fine-tuning achieves a perplexity of $20.0$ and the pre-trained public model achieves a perplexity of $37.5$ . Lower perplexity is better.
\label{fig:summary_main}}

\vspace{-10pt}
  \end{center}
\end{wrapfigure}

\paragraph{Baseline Comparisons} We first compare \textsc{SubMix} to three privacy-preserving mechanisms as baselines: (1) DP-SGD~\citep{abadi2016deep} for private training; and (2) subsample-and-aggregate (S\&A, \citet{dwork2014algorithmic}) and (3) GNMax~\citep{papernot2018scalable} for private prediction.  For DP-SGD, concurrent work~\citep{li2021large} proposed training with very large batch size and fixed number of updates to significantly improve privacy-utility trade-off. The drawback of this method is that it requires taking multiple passes of the data in order to reap the benefits (whereas only one pass is needed for SubMix and non-private training). See \autoref{sec:supp_details} for details on adapting these mechanisms for private next-token prediction.
\autoref{fig:summary_main} shows the predictive perplexity of the private mechanisms on the \texttt{Wikitext-103} dataset for ROP/RDP\footnote{Privacy loss is measured under RDP for  DP-SGD, S\&A, and GNMax; and under ROP for \textsc{SubMix}.} parameters $\alpha=2$ and $\epsilon=2$. The pre-trained GPT-2 model has a perplexity of 37.5 on \texttt{Wikitext-103} and is trivially private on that corpus. Fine-tuning the LM non-privately achieves a perplexity of 20.0. \textsc{SubMix} achieves a perplexity of 26.9 at $B=1,024$ queries, which is substantially below the perplexity of the pre-trained LM. By contrast, the other mechanisms do not improve over the pre-trained baseline, even for a \emph{single} query ($B=1$).
\autoref{sec:supp_baselines} presents more detailed results on the privacy-utility trade-off of the baseline methods: all of them require extremely large $\epsilon$ to improve over the pre-trained baseline, with DP-SGD outperforming the private prediction baselines (S\&A and GNMax).



\begin{figure}[t]
\begin{subfigure}{.5\textwidth}
  \centering
  \includegraphics[width=\linewidth]{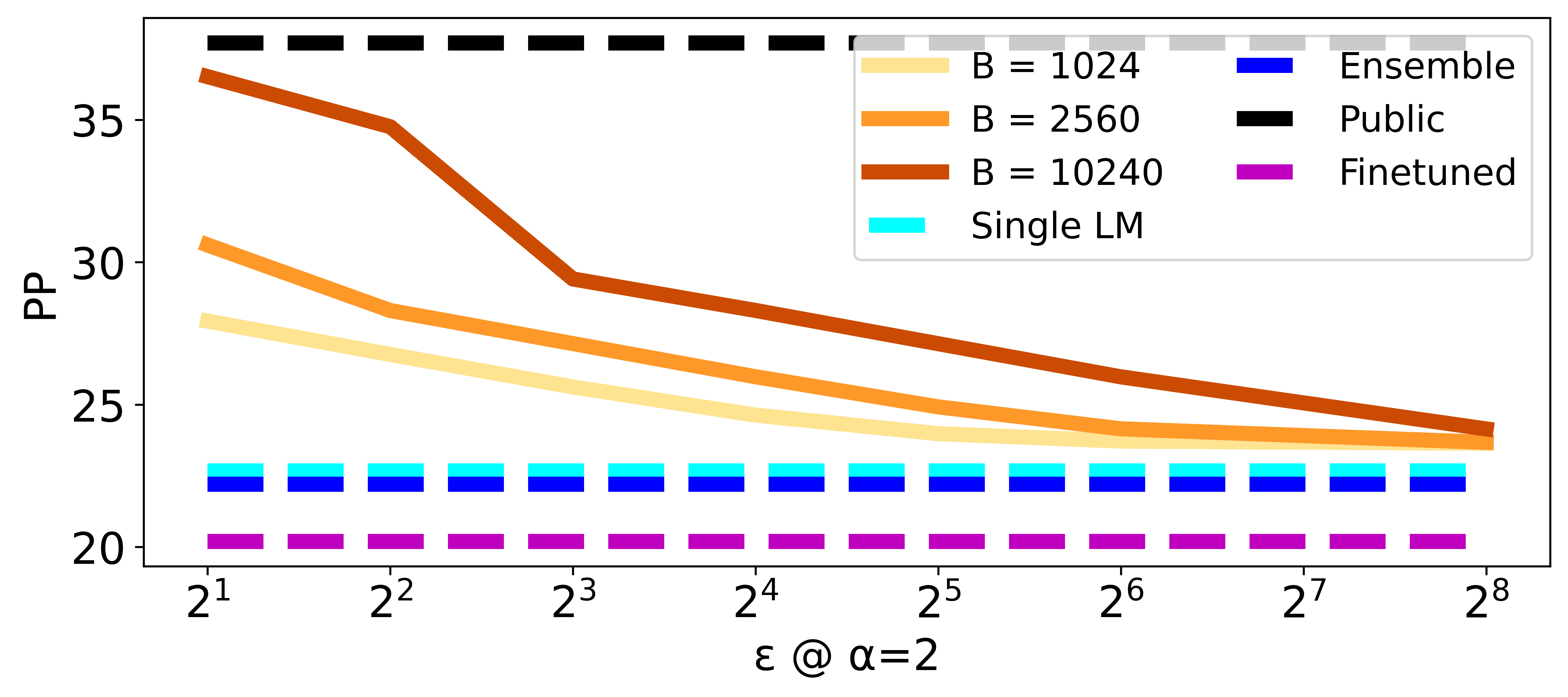}  
\end{subfigure}
\begin{subfigure}{.5\textwidth}
  \centering
  \includegraphics[width=\linewidth]{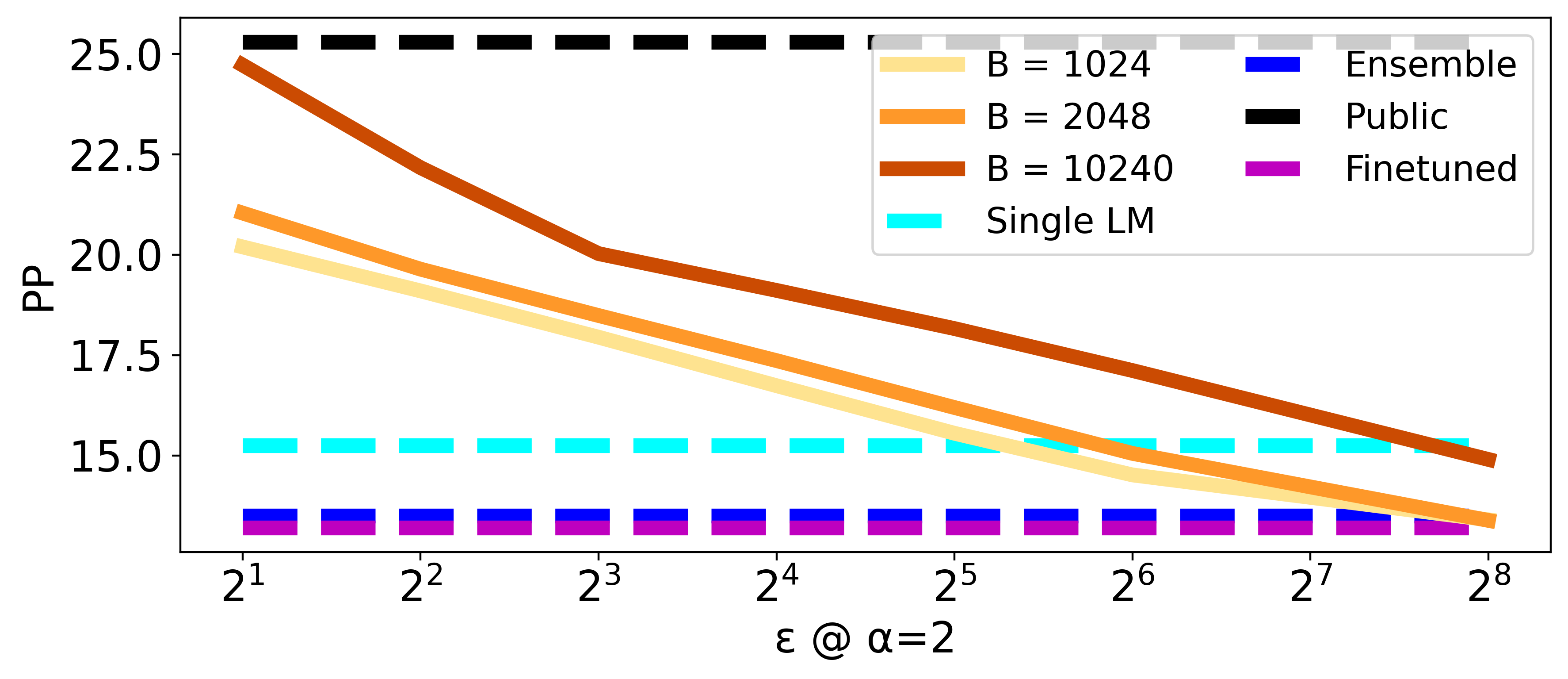}  
  \end{subfigure}

\caption{Perplexity of \textsc{SubMix} ($k = 8$) on \texttt{Wikitext-103} (\textbf{left}) and \texttt{BigPatent-G} (\textbf{right}) as a function of ROP privacy loss $\epsilon$ for three query budget values $B$. The perplexity of the pre-trained model and (non-private) single model, ensemble, and fully fine-tuned models are shown for reference.}
\label{fig:ppl_vs_epsilon}
\end{figure}

\begin{figure}[t]
\begin{subfigure}{.5\textwidth}
  \centering
  \includegraphics[width=\linewidth]{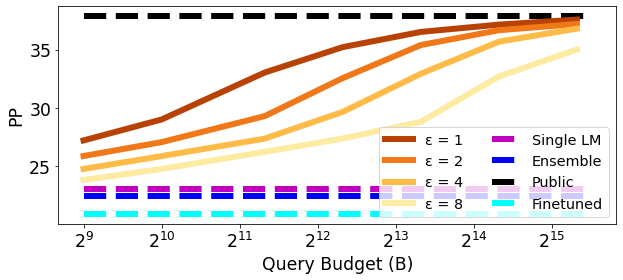}  
\end{subfigure}
\begin{subfigure}{.5\textwidth}
  \centering
  \includegraphics[width=\linewidth]{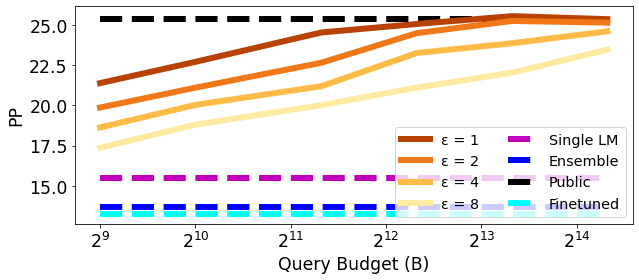}  
\end{subfigure}

\caption{Perplexity of \textsc{SubMix} ($k = 8$) on \texttt{Wikitext-103} (\textbf{left}) and \texttt{BigPatent-G} (\textbf{right}) as a function of query budget $B$ for different ROP privacy losses $\epsilon$. The perplexity of the pre-trained model and (non-private) single model, ensemble, and fully fine-tuned models are shown for reference.}
\label{fig:ppl_vs_budget}
\end{figure}

\paragraph{Varying the Privacy Loss} In \autoref{fig:ppl_vs_epsilon}, we show the trade-off between perplexity and ROP privacy loss, $\epsilon$, of \textsc{SubMix} at $\alpha=2$. On both \texttt{Wikitext-103} (left plot) and \texttt{BigPatent-G} (right plot), \textsc{SubMix} substantially improves over the pre-trained GPT-2 model, even when the query budget is increased to $B=10,240$ queries. As expected, \textsc{SubMix} matches the perplexity of non-private LM at higher values of $\epsilon$. 
Interestingly, \textsc{SubMix}'s perplexity is even lower than that of a single non-privately fine-tuned LM at high $\epsilon$.
We surmise this is due to the performance gap between a single fine-tuned LM and an LM ensemble. 
The effect is less pronounced on \texttt{Wikitext-103} because that corpus was split into users by block, as a result of which many LMs contain text blocks from the same Wikipedia article. This reduces the positive effects of ensembling on predictive perplexity. 

\paragraph{Varying the Query Budget} \autoref{fig:ppl_vs_budget} shows the trade-off between perplexity and the number of queries $B$. The results in the figure were obtained by tuning the target leakage to obtain the desired budget. As expected, answering more queries using \textsc{SubMix} increases the average perplexity for each next-token query at a given $\epsilon$. However, \textsc{SubMix} attains a surprisingly low perplexity for a moderate number of queries (\emph{e.g.}, $B=2^{10}$) at all $\epsilon$ values on both \texttt{Wikitext-103} and \texttt{BigPatent-G}.


\begin{wrapfigure}{r}{0.55\textwidth}
  \begin{center}
\vspace{-20pt}
  \includegraphics[width=\linewidth]{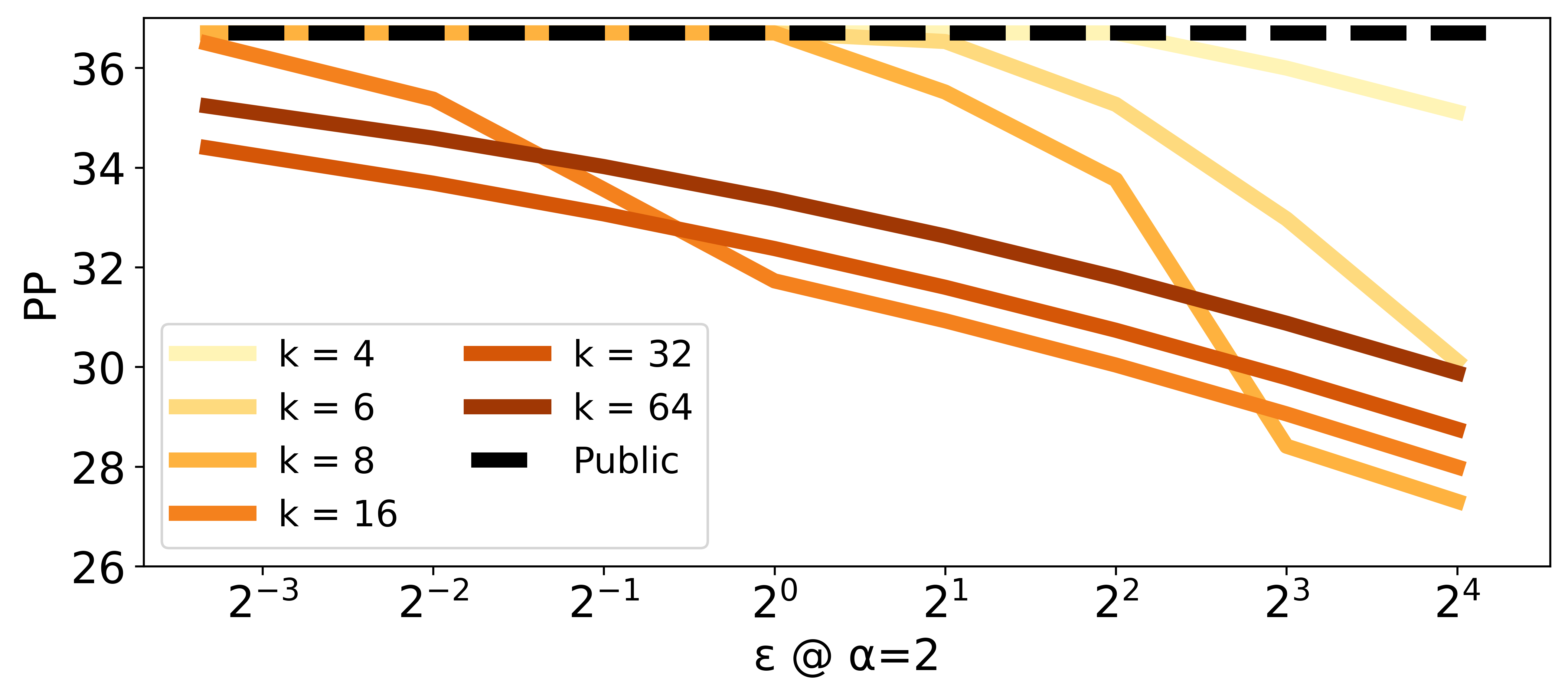} 
\vspace{-15pt}
\caption{\small Perplexity of \textsc{SubMix} on \texttt{Wikitext-103} as a function of ROP privacy loss $\epsilon$ for different partition sizes $k$. \label{fig:ppl_vs_numparts}}

\vspace{-10pt}
  \end{center}
\end{wrapfigure}

\paragraph{Varying the Number of Parts} A key hyperparameter of interest in \textsc{SubMix} is the size of the partition $\Pi$. Intuitively, a smaller number of partitions, allows each part to train a better quality LM at the cost of a larger R\'{e}nyi divergence when a part is removed. \autoref{fig:ppl_vs_numparts} shows the trade-off between perplexity and ROP privacy loss, $\epsilon$, for varying partition sizes, $k$.
We observe the key trend that as $\epsilon$ decreases, perplexity increases more rapidly for smaller $k$ because the privacy budget is exhausted more quickly when each part has a greater relative contribution to the ensemble. The optimal value for $k$ is generally around 16 for all $\epsilon$ values  of interest, although this may depend on the data distribution and design choices such as model architecture and training hyperparameters.


\subsection{Text Extraction Attacks}

To empirically validate that \textsc{SubMix} prevents text extraction attacks, we perform a random code text-extraction experiment in the style of \citet{ramaswamy2020training,shi2021selective}. We randomly generate $m$ codes with each code being an $\ell$-digit number (for example, representing a user's age, ZIP-code, phone number, SSN, \emph{etc.}). 
The fine-tuning dataset is constructed so that each user's text is single sentence: $\mathcal{D}_i = $ ``\texttt{My number is: <random $\ell$-digit number here>}''. We then fine-tune on this dataset and make private predictions using \textsc{SubMix} for the query context ``\texttt{My number is:}'' to test whether \textsc{SubMix} prevents text extraction (and if so, at what $\epsilon$).
As a baseline, we fine-tune GPT-2 on this dataset for $1000$ iterations. This results in the LM memorizing all $\ell$ codes, achieving a perplexity of less than $0.5$ and $\geq 90\%$ recall when prompted with context. We apply \textsc{SubMix} with $k = \nicefrac{m}{2}$ parts so that each model in the ensemble strongly memorizes one number, achieving near $100\%$ recall when prompted with the context.

\begin{wrapfigure}{r}{0.55\textwidth}
  \begin{center}
\vspace{-15pt}
  \includegraphics[width=\linewidth]{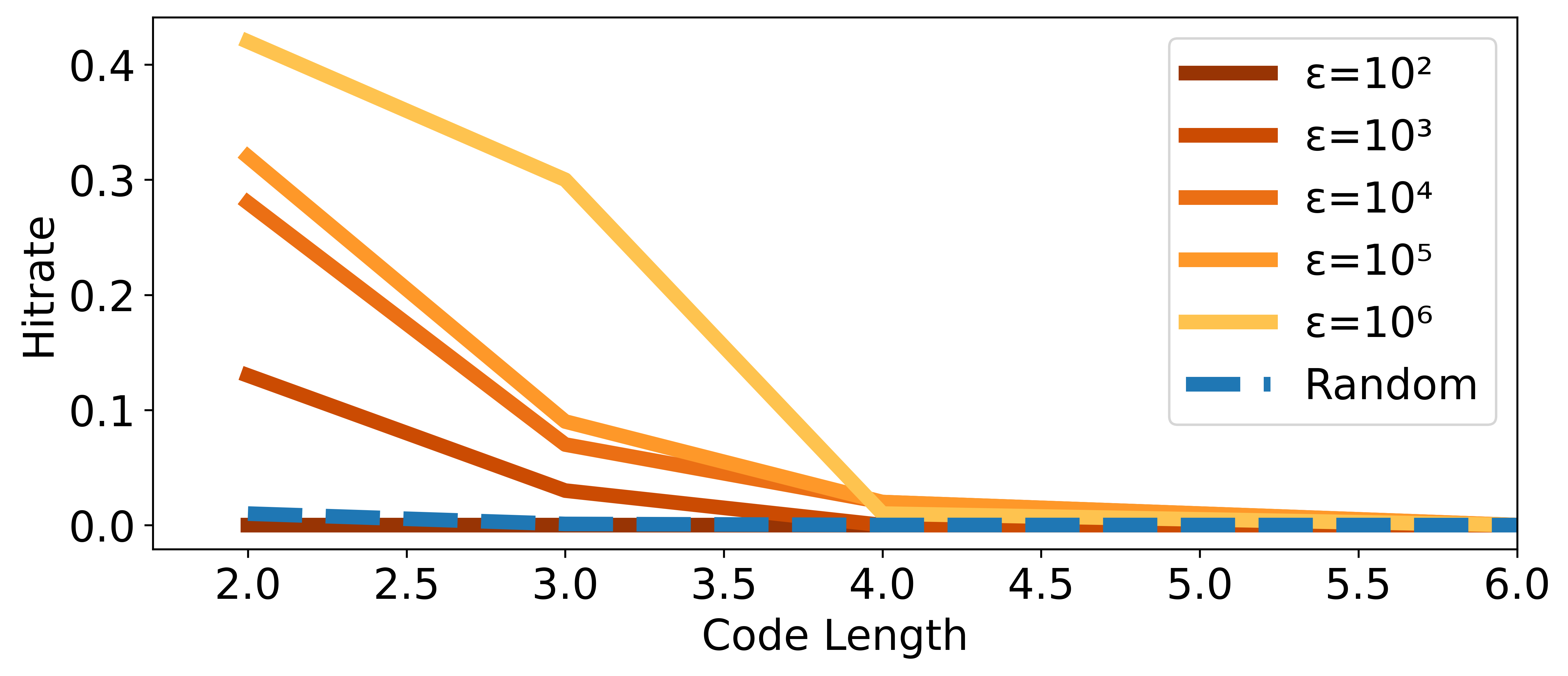} 
\vspace{-15pt}
\caption{\small Hit rate of text extraction attacks on \textsc{SubMix} for varying lengths of code. The \# of parts is $k = 3$ and the \# of codes generated is $g=100$. The non-privately fine-tuned LM has a hit rate of $\geq 0.9$  for all lengths. \label{fig:text_extract_main}}

\vspace{-10pt}
  \end{center}
\end{wrapfigure}

For the text extract attack, the figure-of-merit is the \emph{hit rate} of the $g$ generations, defined as the number of generated codes that exactly match a secret code divided by $g$. \autoref{fig:text_extract_main} shows the hit rate of the text extraction attack. For all code lengths $\ell=2,\ldots,5$, \textsc{SubMix} succeeds in preventing the attack at $\epsilon=10^2$. For longer code lengths, even higher values of $\epsilon$ suffice for preventing this random-sampling text extraction. Intuitively, extraction becomes more difficult as the code lengths increase. This experiment shows that the mechanism for limiting the release of sensitive information via solving \autoref{eq:opt_lambda} is effective, and the privacy accounting in \textsc{SubMix} meaningfully measures the amount of privacy loss. We also ran these attacks for varying values of $K$ and $m$, and made the same qualitative observations. In addition, we performed experiments in which we varied $g \in \{10,100,1000\}$. We found that this does not affect \textsc{SubMix}'s hit rate. 

\section{Discussion \& Related Work}

\paragraph{Related Work} \citet{mcmahan2017learning} was the first to study the training of differentially private language models by using DP-SGD to train a small recurrent neural network. However, the resulting LMs have far fewer parameters than modern transformers and attain much higher perplexities. More recently, \citet{shi2021selective} explored an alternative approach called \emph{selective differential privacy}, where the privacy guarantee only applies to blocks of text in the training corpus that are deemed sensitive, \emph{e.g.}, addresses and phone numbers. Unfortunately, this approach is difficult to scale to large unstructured text corpora because it requires annotating all text in the corpus with a privacy sensitivity level.

\textsc{SubMix} has conceptual similarities to PATE~\citep{papernot2016semi, papernot2018scalable} for private semi-supervised learning. Both \textsc{SubMix} and PATE make predictions using an ensemble of models trained (or fine-tuned, in the case of \textsc{SubMix}) on private data, and employ a data-dependent and query-dependent privacy accounting mechanism at prediction time. The central idea in both methods is that privacy loss is small when models trained on different parts of the data agree on a prediction. However, PATE is more natural in discriminative or classification tasks because it return a distribution's noisy argmax. On the other hand, \textsc{SubMix} is more natural in generative tasks because it returns a sample from the distribution. 


\textsc{SubMix} is also related to prior work on private posterior sampling \citep{geumlek2017r, dimitrakakis2017differential}, where the randomness in the privacy mechanism comes from releasing a sample from a distribution defined by the private data. In particular, \textsc{SubMix} uses a privacy accounting methodology based on R\'{e}nyi divergences similar to that of \citet{geumlek2017r}.

\cite{mireshghallah2021privacy} propose adding a privacy regularizer to language model training to reduce its memorization of sensitive text. They empirically showed that the regularizer reduces the effectiveness of text extraction attacks, but it does not satisfy any formal privacy guarantee such as DP. Other related works exploring privacy in natural language processing include \citep{gopi2020differentially,lyu2020differentially,xu2020differentially,li2018towards,kim2021differentially}.

\paragraph{Limitations \& Future Directions} One limitation of the \textsc{SubMix} protocol as presented here is that it only supports decoding from the ensemble of LMs by sampling directly from the predicted pmf. This type of sampling is known to produce unnatural and incoherent text \citep{kulikov2018importance, holtzman2019curious}. Better decoding methods such as top-k sampling \citep{fan2018hierarchical} and nucleus sampling \citep{holtzman2019curious} exist, but they require  modifications to the protocol that may cause additional privacy leakage. However, we note that a close alternative to top-k and nucleus sampling is \emph{temperature decoding} \citep{holtzman2019curious}, which scales the predicted pmf by a temperature term to decrease its entropy. \textsc{SubMix} readily supports this decoding method by applying temperature scaling as a post-processing step. In future work, we aim to extend out work by designing protocols that can support different types of decoding strategies as well.



Another limitation of \textsc{SubMix} is that the use of an ensemble substantially increases the computational and storage requirements.
Our experiments suggests that an overhead factor of $8$ is needed to attain a non-vacuous trade-off between privacy and utility. One potential solution to reduce the computational requirements of \textsc{SubMix} may be to fine-tune only the top few transformer layers closest to the prediction head. This would allow the evaluation of the bottom transformer layers to be shared between models in the ensemble, thereby reducing both computation and storage requirements. We leave the exploration of such efficiency improvements for future work.

\bibliography{references}  
\bibliographystyle{iclr2022_conference}
\newpage
\appendix


\section*{\Large Appendix}

\section{Supplementary Figures}
\label{sec:supp_figures}

\subsection{Baselines}
\label{sec:supp_baselines}

\begin{figure}[h]
     \centering
     \hfill
     \begin{subfigure}[b]{0.45\textwidth}
         \centering
         \includegraphics[width=\textwidth]{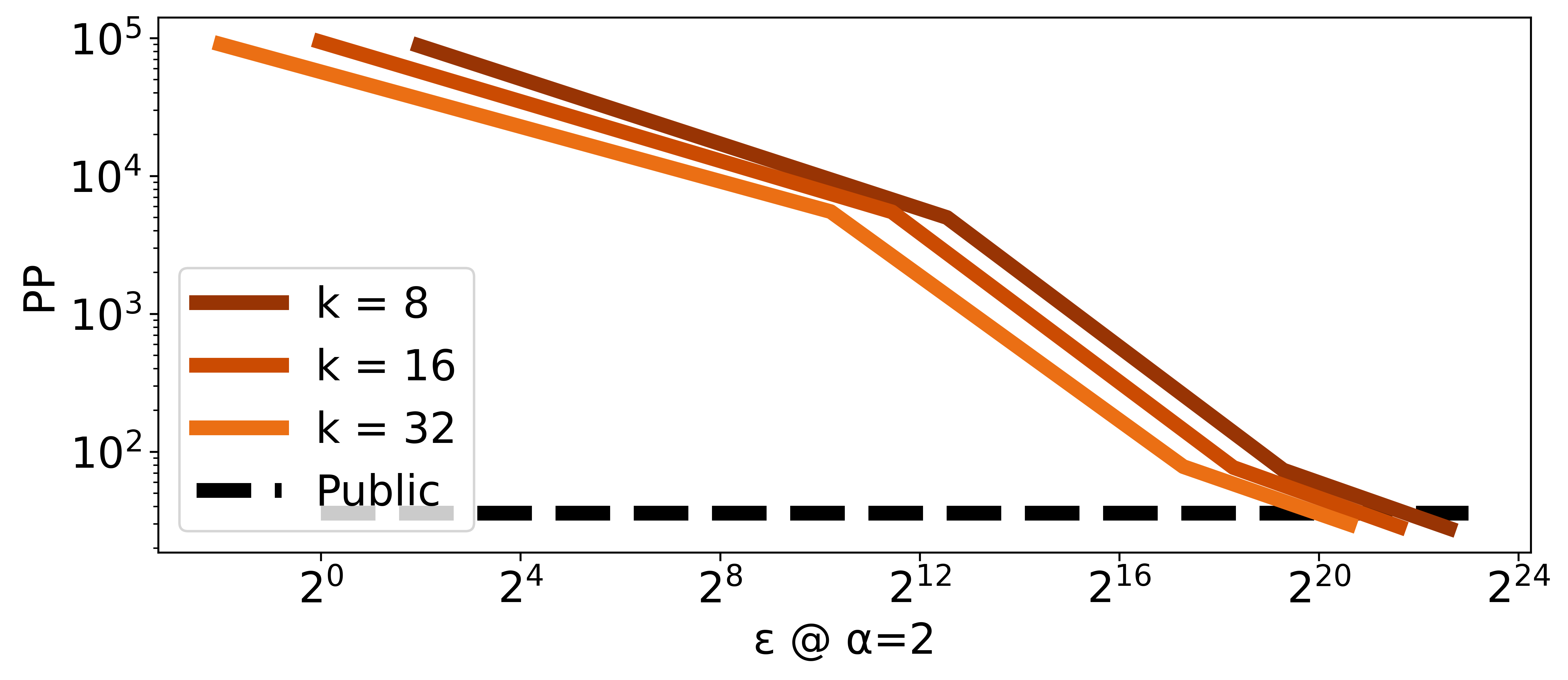}
         \caption{Subsample-and-aggregate}
         \label{fig:ppl_eps_sa}
     \end{subfigure}
     \hfill
     \begin{subfigure}[b]{0.45\textwidth}
         \centering
         \includegraphics[width=\textwidth]{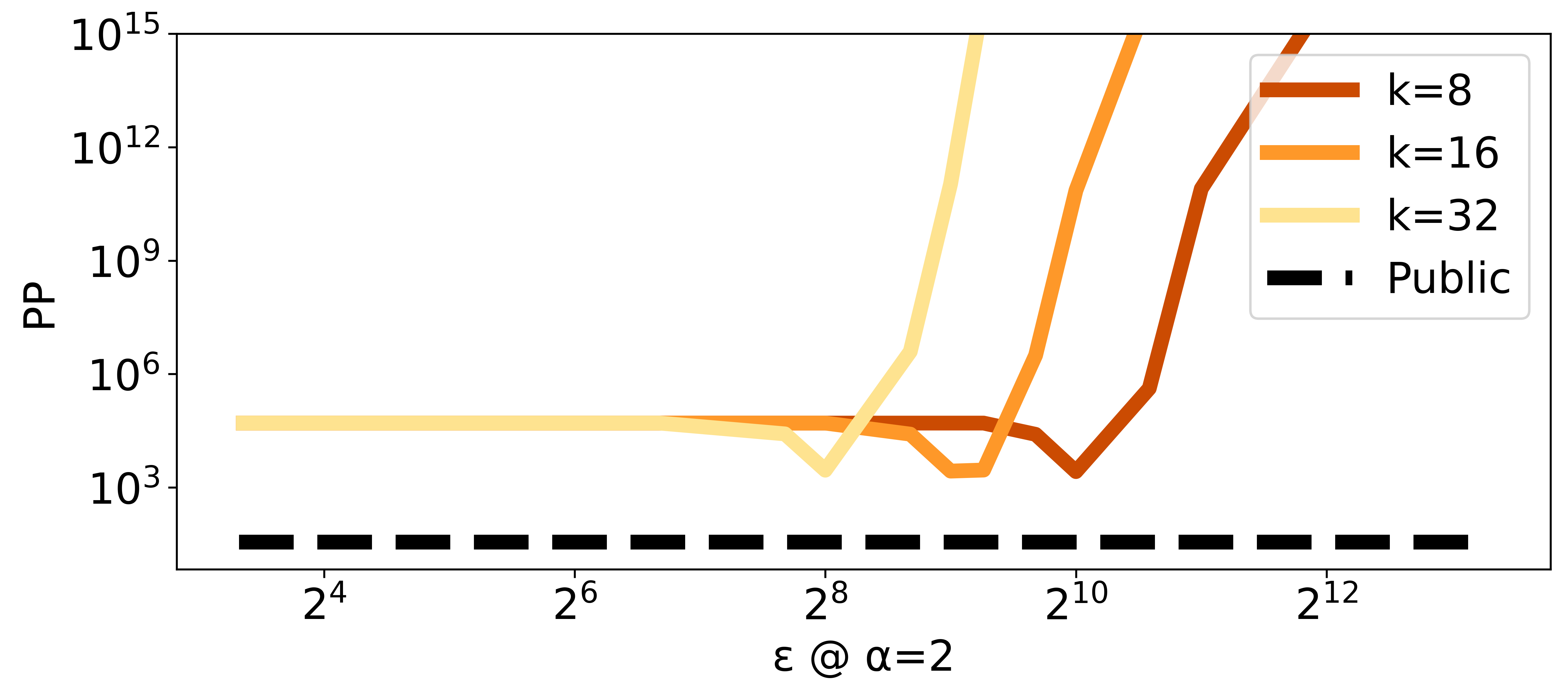}
         \caption{GNMax}
         \label{fig:ppl_eps_gnmax}
     \end{subfigure}

\label{fig:three graphs}
\caption{Perplexity of general-purpose private prediction mechanisms at different values of RDP privacy parameter $\epsilon$ on \texttt{Wikitext-103}.}
\vspace{-15pt}
\end{figure}


We perform a more comprehensive evaluation of the privacy-utility trade-off for the baseline mechanisms. We observe that subsample-and-aggregate (S\&A) in \autoref{fig:ppl_eps_sa} does not achieve a favorable trade-off. We also vary the number of parts $k$ for the ensemble.
\autoref{fig:ppl_eps_gnmax} shows the privacy-utility trade-off for the GNMax baseline. Unlike DP-SGD and S\&A, GNMax achieves its minimal perplexity at some value of $\epsilon$ rather than perplexity being monotonically decreasing in $\epsilon$. This is due to the fact that there is a reasonably strong consensus amongst the LM ensemble. When the noise magnitude is too small, the ensemble concentrates its prediction onto a single token, hence having an unbounded perplexity on other likely tokens. When the noise magnitude is too large, GNMax converges to the uniform distribution over all tokens. Nevertheless, even at the empirically optimal noise magnitude, GNMax achieves a worse perplexity than the public pre-trained LM. 

\subsection{R\'{e}nyi Divergence Order}
\label{sec:supp_renyi}

\begin{figure}[h]
\centering
\includegraphics[width=.5\textwidth]{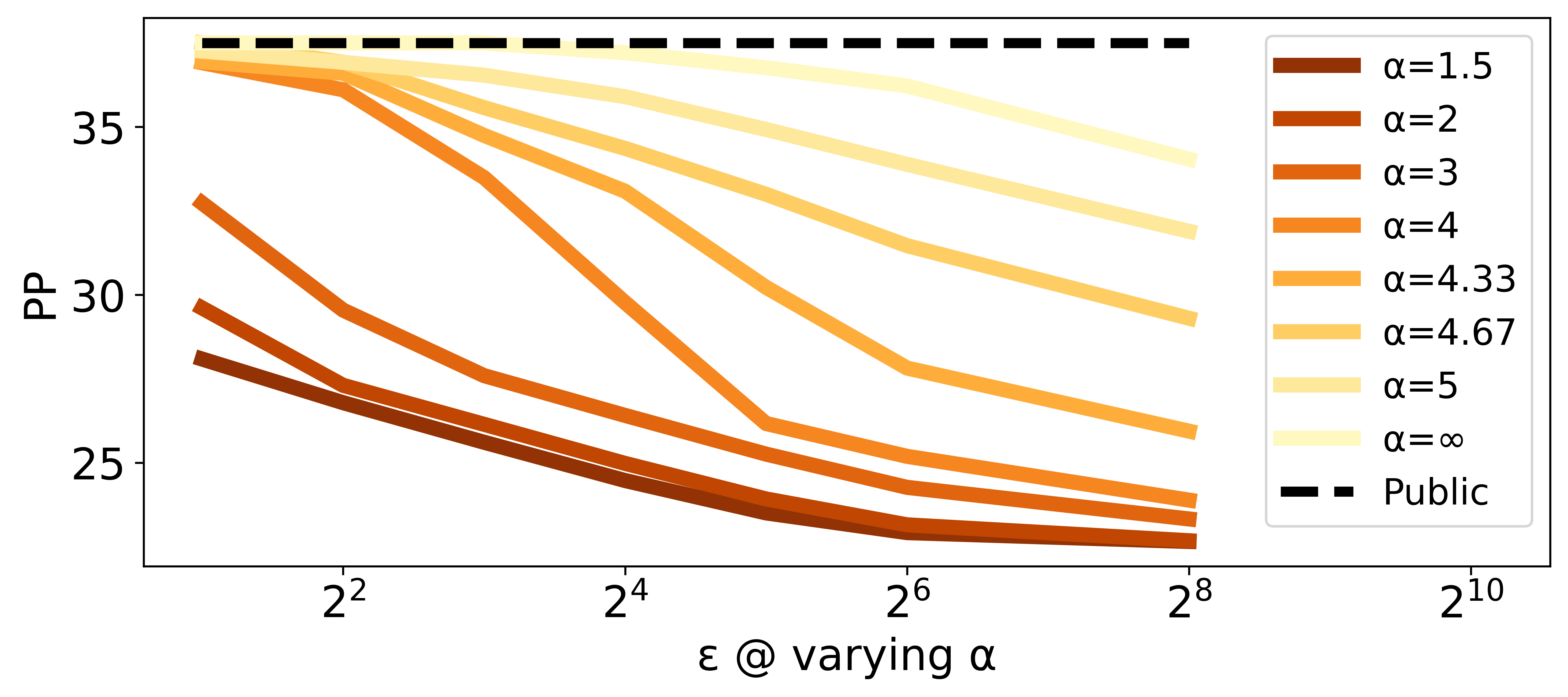}
\caption{Perplexity of \textsc{SubMix} ($k=8$) on \texttt{Wikitext-103} as a function of ROP privacy loss $\epsilon$ for different values of R\'{e}nyi divergence order $\alpha$.  Budget $B = 2560$.}
\label{fig:renyi_order}
\end{figure}



In our main experimental results, we used a R\'{e}nyi divergence order of $\alpha =2$. We explore how \textsc{SubMix}'s privacy-utility trade-off varies under different choices of $\alpha$. Note that R\'{e}nyi divergence is monotonically increasing in $\alpha$~\citep{mironov2017renyi}, so the RDP guarantee further worsens for baseline mechanisms at higher $\alpha$.
\autoref{fig:renyi_order} shows the perplexity attained by \textsc{SubMix} at different ROP privacy parameters $\epsilon$ while varying the R\'{e}nyi divergence order $\alpha$.
We observe that \textsc{SubMix} retains its desirable privacy-utility trade-off for $\alpha \leq 4$. At $\alpha > 4$, we observe a sharp decline in the performance of \textsc{SubMix} even at high values of $\epsilon$. 
However, note that for all values of $\alpha$, \textsc{SubMix} remains non-vacuous by achieving a non-negligible reduction in perplexity compared to the pre-trained LM at $\epsilon = 256$, although this is far from a reasonable guarantee.


\section{Experimental \& Implementation Details}
\label{sec:supp_details}

We provide additional details about \textsc{SubMix} and the baseline private prediction mechanisms.

\subsection{\textsc{SubMix}}

For \textsc{SubMix} training, on both \texttt{Wikitext-103} and \texttt{BigPatent-G}, we fine-tune GPT-2 using HuggingFace's AdamW optimizer at a learning rate of $0.0001$ and batch size of $8$. We use a linear warm-up schedule with the number of steps equal to the number of batches. At prediction time, we set the target privacy leakage parameter $\beta$ using the simple yet solid heuristic $\beta = \epsilon/B$, where $B$ is the number of queries to be answered.



\subsection{DP-SGD}

Concurrently with this work, \cite{li2021large} developed improved hyperparameter configurations for DP-SGD training of large-scale transformers. We have improved our DP-SGD baseline to reflect those findings. Of course, DP-SGD is a training algorithm, not a prediction protocol, and thus the released model can be used indefinitely. On the other hand, one of the major drawbacks of the DP-SGD is that it requires a large compute overhead. We find that for a single pass of the data, DP-SGD only marginally improves upon the pre-trained baseline at small $\epsilon$ on Wikitext-103, even using the newly discovered configurations. Allowing for a $10\times$ increase in computation (i.e. 10 total epochs), we do see modest improvements with DP-SGD. It is possible that more epochs of training could further improve DP-SGD, but we leave that exploration for future work.

We use the PyTorch Opacus\footnote{\url{https://github.com/pytorch/opacus}} library for DP-SGD training. We swept in the range $[0.1,1]$ for the noise multiplier to obtain the target $\epsilon$. We set a small gradient clipping norm in of $0.1$, as suggested by \cite{li2021large}. We swept large batch sizes of $\{2048,4096,8192 \}$, as suggested by \cite{li2021large}. We used the AdamW solver with a learning rate of $0.0001$. 

\subsection{Subsample-and-aggregate}

Similar to \textsc{SubMix}, we train the LM ensemble for S\&A by splitting the private corpus into $k$ parts and training an LM on each part using identical hyperparameters.
At prediction time, each LM outputs a next-token pmf $h_i(\mathbf{x}_t)$, and the mechanism aggregates these pmfs using a simple average: $h(\mathbf{x}_t) = \nicefrac{1}{k} \sum_{i=1}^k h_i(\mathbf{x}_t)$.
Since each $h_i(\mathbf{x}_t)$ is a probability vector with $L_1$-norm equal to 1, the global $L_1$-sensitivity of each LM's prediction is $1/K$.
Based on this sensitivity, we apply the Laplace mechanism with a suitable scale to obtain R\'{e}nyi differential privacy. Similar to \textsc{SubMix}, we compose the leakage over multiple queries using adaptive composition. Finally, we sweep the Laplace noise scale to obtain the optimal privacy-utility trade-off. 

\subsection{GNMax}

GNMax trains an LM ensemble in the same manner as \textsc{SubMix} and S\&A. At prediction time, GNMax produces a next-token histogram $\bar{\mathbf{n}}$, where $\bar{\mathbf{n}}_x$ is equal to the number of LMs in the ensemble that predict $x$ as its top next token. The mechanism then adds Gaussian noise with scale $\sigma$ to $\bar{\mathbf{n}}$ and outputs the token with the highest count after noise addition. Privacy accounting is done using the data-dependent and query-dependent bound in Theorem 6 of \citet{papernot2018scalable}. Similar to \textsc{SubMix} and S\&A, we accumulate privacy loss across multiple queries using adaptive composition.

Computing the predictive perplexity of GNMax is not straightforward as the mechanism does not output a next-token pmf, and the induced next-token pmf via taking the argmax of $\bar{\mathbf{n}} + \mathcal{N}(0, \sigma^2)$ does not have a simple closed form. Instead, we \emph{lower bound} the perplexity using the following inequality. Let $\mathbf{q}$ denote the induced next-token pmf for GNMax, and let $x$ be the ground truth next token. The (log) perplexity for this prediction is equal to $-\log(\mathbf{q}_x)$. Let random variable $N$ denote the argmax of $\bar{\mathbf{n}} + \mathcal{N}(0, \sigma^2)$.
We can upper bound $\mathbf{q}_x$ with:
\begin{equation*}
    \mathbf{q}_x = \mathbf{Pr}[x = N] \leq   \mathbf{Pr}[x = N|N \in \{z: \bar{\mathbf{n}}_z \geq \bar{\mathbf{n}}_x  \}] \leq \frac{1}{|\{z: \bar{\mathbf{n}}_z \geq \bar{\mathbf{n}}_x  \}|}
\end{equation*}
The first inequality follows from the fact that conditioning $N$ onto a subset of outcomes does not decrease the probability of any outcome that satisfies the condition, and the second inequality follows from the fact that $\mathbf{Pr}[x = N]$ is monotonically increasing in $\bar{\mathbf{n}}_x$.
We also have the upper bound:
\begin{equation*}
    \mathbf{q}_x \leq \mathbf{Pr}[ \mathcal{N}(\bar{\mathbf{n}}_x - \bar{\mathbf{n}}_{z^*}, 2\sigma^2) \geq 0 ],
\end{equation*} 
where $z^* = \text{argmax}_{z \neq x}\bar{\mathbf{n}}_z$. Taking the negative log of the minimum of the two upper bounds above allows us to lower bound the perplexity. We reported this lower bound in all of our experiments instead of the true perplexity for $\mathbf{q}$.

\section{Differential Privacy Conversions}
\label{sec:supp_conversion}

We give conversion for RDP mechanisms using partition-level adjacency compared to user-level adjacency. Theoretically speaking, partition-level adjacency is neither stronger nor weaker than user-level adjacency. Therefore, conversion is costly in both direction. With that being said, at larger $\alpha$, the conversion is a modest factor of $2$ for partition-to-user, whereas the conversion cost grows large for small $\alpha$. The conversion cost for user-to-partition does not depend on $\alpha$ and is always a factor of $n/k$, which is generally intolerably large.  

\subsection{Conversion from Partition-level RDP to User-level RDP}
\begin{theorem}
For $\alpha > 2$, if a mechanism is partition-level $(\alpha, \epsilon)$-RDP under any partition, then it is user-level $(\alpha/2, \frac{2\alpha - 3}{\alpha -2 }\epsilon)$-RDP
\label{thm:conversion_adj}
\end{theorem}
\begin{proof}
This follows from the application of the weak triangle inequality for Renyi divergence (Corollary 4a in \cite{mironov2017renyi}).

Let $M$ be the mechanism's output, $M'$ be the mechanism's output with arbitrary change to the text of the $i$-th user, and let $M''$ be the mechanism's output with the entire part containing the $i$-user removed. Recall that in the main text we defined the \emph{symmetrized} Renyi divergence $D_{\alpha}^{\text{sym}}(P||Q) = \max\{ D_{\alpha}(P||Q), D_{\alpha}(Q||P) \}$.

The proof proceeds by showing that $M$ and $M'$ are always in a \emph{symmetric} Renyi divergence ball about $M''$. In turn, this upper bounds the Renyi divergence between $M$ and $M'$.

The weak triangle inequality tells us that for any $R$:
$$D_\alpha(P||Q)  \leq \frac{\alpha - 1/2}{\alpha -1 }D_{2\alpha}(P||R) + D_{2\alpha-1}(R||Q) $$

Plugging in $P \gets M$, $Q \gets M'$ and $R \gets M''$ yields the result with little effort.

$$D_{\alpha/2}(M||M')  \leq \frac{\alpha/2 - 1/2}{\alpha/2 -1 }D_{\alpha}(M||M'') + D_{\alpha-1}(M''||M') $$
$$ \leq \frac{\alpha/2 - 1/2}{\alpha/2 -1 }D_{\alpha}(M||M'') + D_{\alpha}(M''||M')$$
$$ \leq \frac{\alpha/2 - 1/2}{\alpha/2 -1 }D_{\alpha}^{\text{sym}}(M||M'') + D_{\alpha}^{\text{sym}}(M'||M'')$$
$$ \leq \frac{\alpha/2 - 1/2}{\alpha/2 -1 }D_{\alpha}^{\text{sym}}(M||M'') + D_{\alpha}^{\text{sym}}(M'||M'')$$
$$ \leq \frac{\alpha - 3/2}{\alpha/2 -1 }\epsilon$$
\end{proof}

Depending on $\alpha$ and $\epsilon$, the conversion can be tightened somewhat by using the more general version of the weak triangle inequality (Prop. 11 in \cite{mironov2017renyi}).

\begin{corollary}
If a mechanism is partition-level $\epsilon$-DP under any partition, then it is user-level $(2\epsilon)$-DP.
\end{corollary}
\begin{proof}
Follow the proof in Thm. \ref{thm:conversion_adj} but in the limit as $\alpha \rightarrow \infty$. Notice that in this case, $D_{\infty}^{\text{sym}}$ is a proper metric and there is no slack in the triangle inequality.
\end{proof}

\subsection{ Conversion from User-level RDP to Partition-level RDP}

\begin{theorem}
If a mechanism is user-level $(\alpha, \epsilon)$-RDP then it is partition-level $(\alpha, \frac{n}{k} \epsilon)$-RDP for a uniform $k$-partition 
\end{theorem}

\begin{proof}
Follows from the conversion from user-level RDP to group-level RDP (see a standard refernce, for example \cite{dwork2014algorithmic}). For group size of $n/k$, group-level RDP implies partition-level RDP under a uniform partition since group-level RDP implies statistical indistinguishability up to any choice of $n/k$ users.
\end{proof}

\subsection{Implications for Correlation Attacks}
One reason for using group (R)DP is protection again correlation attacks. In the case that private text might exist amongst more than one user, such as a shared secret, group differential privacy provides protection as long as the number of users sharing the secret is small than the group size parameter. In general, $\epsilon$-DP and $(\alpha, \epsilon)$-RDP guarantees implicitly a group size of $k = 1$. Naively, such a guarantee can be converted to a larger group size for a cost \emph{multiplicative} in the group size. A user-level $(\alpha, \epsilon)$-RDP mechanism is also group-level $(\alpha, \kappa\epsilon)$-RDP under group size $\kappa$.  

Therefore, any RDP guarantee \emph{does} provide some protection against correlation attacks, but the strength of the protection decreases rapidly as the number of correlated user increases.

In some sense, this is to be expected, since, obviously, privacy becomes infeasible when all or a majority of users are correlated. However, for a modest number of correlated users, the privacy guarantee can still be significant.

\subsection{Conversion from Variable to Fixed Length Prediction Sequences}
Our definition of $(\alpha, \epsilon)$-ROP assumes a variable-length query sequence. We discuss here how to convert any variable-length $(\alpha, \epsilon)$-ROP guarantee into a fixed-length one. To clarify this distinction, we write $(\alpha, \epsilon, T)$-ROP to imply that the sequence length is fixed at $T$. Note that fixed-length $(\alpha, \epsilon, T)$-ROP differs from $(\alpha, \epsilon, T)$-RDP only in the notion of adjacency. 

In order to do this, we make use of a general-purpose \emph{random stopping} (RS) mechanism. Recall that a private prediction protocol $\mathcal{P}$ can issue a termination signal at any time, under the condition that any future queries must be answered in a data-independent way. We will use the public pre-trained model $h_\emptyset$ to answer queries after termination. 
The random stopping mechanism is parameterized by a \emph{fixed} response budget $B$ and an expansion factor $C > 1/2$. The random stopping mechanism then uniformly at random selects some value $\tau \in \{1,...,CB\}$ and issues the termination before the $\tau$-th response is made, \emph{if it has not been issued already}. We refer to such a mechanism as a $(B,C)$-random stopping (RS) mechanism.

The proposition below shows that the additional information leaked from $T(\mathcal{P})$ for a $(B,C)$-RS mechanism is at most $\log\left({CB}\right)$. Furthermore, we show that if the queries are drawn iid from some test distribution, then the expected test perplexity for this fixed-length version of the private prediction mechanism can be derived exactly.

\begin{proposition}
Let $h_\mathcal{\emptyset}$ be a public LM. The following are true:

(1) The $(B,C)$-$\mathsf{RS}$ mechanism converts any $(\alpha, \epsilon)$-ROP prediction protocol $\mathcal{P}$  to an  $(\alpha, \epsilon + \log (CB), B)$-RDP prediction protocol.

(2) Suppose that the public LM achieves an expected test perplexity of $p_\emptyset$. If $\mathcal{P}$ has a stationary pre-termination expected test perplexity of $p$, then the expected test perplexity post-conversion is upper bounded by $(1-\frac{1}{2C})p + \frac{p_\emptyset}{2C}$
\end{proposition}

\begin{proof}

\emph{Part (1)} Let $\mathcal{P}$ be an $(\alpha, \epsilon)$-ROP prediction prediction protocol. Let $\mathbf{T}$ be the set of all termination rules that may (causally) depend on $\mathcal{D}$, $\Pi$, and $\{\mathbf{x}_t\}$. Let $\mathcal{T}$ denote the particular choice of termination rule used by $\mathcal{P}$. Let $\mathsf{RS}(\mathcal{T})$ denote rule $\mathcal{T}$ wrapped with the $(B,C)$-RS mechanism.
 
 Consider the random length $B$ sequence $\mathcal{P} \underset{\scriptscriptstyle{{B}}}{\leftrightharpoons} \mathsf{Adv}$. Recall that after termination is issued, $\mathcal{P}$ falls back to $h_\emptyset$ to answer queries, so no user information can be leaked after termination. Thus, $\mathcal{P} \underset{\scriptscriptstyle{{B}}}{\leftrightharpoons} \mathsf{Adv}$ is information-theoretically equivalent to $\left(\mathcal{P} \underset{\scriptscriptstyle{{\mathsf{RS}(\mathcal{T}))}}}{\leftrightharpoons} \mathsf{Adv}, \mathsf{RS}(\mathcal{T})\right)$. In other words, information leakage can only occur in the head of the sequence (determined by $\mathsf{RS}(\mathcal{T})$) and in the timing of termination $\mathsf{RS}(\mathcal{T})$ itself, but not in the tail in which queries are answered independently of data.
 
 Let $\mathcal{P}'$ and $\mathcal{T}'$ denote adjacent protocol and termination rules, respectively.
 
 We seek to upper bound:
 
 $$ D_{\alpha}\left(\mathcal{P} \underset{\scriptscriptstyle{{\mathsf{RS}(\mathcal{T})}}}{\leftrightharpoons} \mathsf{Adv}, \mathsf{RS}(\mathcal{T})||\mathcal{P}' \underset{\scriptscriptstyle{{\mathsf{RS}(\mathcal{T})}}}{\leftrightharpoons} \mathsf{Adv}, \mathsf{RS}(\mathcal{T}') \right) $$
 
 We must exercise a bit of caution, because $\mathsf{RS}(\mathcal{T})$ is not necessarily independent of  $\mathcal{P} \underset{\scriptscriptstyle{{\mathsf{RS}(\mathcal{T})}}}{\leftrightharpoons} \mathsf{Adv}$, so we cannot split up the joint R\'{e}nyi divergence into a sum without additional justification. 
 

 However, in general for any joint random variable pairs $(X,Y)$ and $(X', Y')$ of equal support:
 
 $$ D_{\alpha}(X,Y|| X',Y') \leq D_\alpha(X||X') + \max_{x \in \mathcal{X}} D_\alpha(Y|x||Y'|x) $$

The above identity follows easily from replacing joint distributions $p(x,y)$ and $p'(x,y)$ with marginalized distributions $p(x)p(y|x)$ and $p'(x)p'(y|x)$ in the definition of R\'{e}nyi divergence.

We apply the identity:

\begin{equation}
D_{\alpha}\left(\mathcal{P} \underset{\scriptscriptstyle{{\mathsf{RS}(\mathcal{T})}}}{\leftrightharpoons} \mathsf{Adv}, \mathsf{RS}(\mathcal{T})||\mathcal{P}' \underset{\scriptscriptstyle{{\mathsf{RS}(\mathcal{T})}}}{\leftrightharpoons} \mathsf{Adv}, \mathsf{RS}(\mathcal{T}') \right) \leq
\label{eqn:prop3_1}
\end{equation}

\begin{equation}
D_{\alpha}\left(\mathcal{P} \underset{\scriptscriptstyle{{\mathsf{RS}(\mathcal{T})}}}{\leftrightharpoons} \mathsf{Adv}||\mathcal{P}' \underset{\scriptscriptstyle{{\mathsf{RS}(\mathcal{T})}}}{\leftrightharpoons} \mathsf{Adv} \right) + \max_{\leftrightharpoons} D_{\alpha}\left( \mathsf{RS}(\mathcal{T})|\leftrightharpoons || \mathsf{RS}(\mathcal{T}') \right)|\leftrightharpoons) \leq 
\label{eqn:prop3_2}
\end{equation} 
 
 \begin{equation} 
 D_{\alpha}\left(\mathcal{P} \underset{\scriptscriptstyle{{\mathsf{RS}(\mathcal{T})}}}{\leftrightharpoons} \mathsf{Adv}||\mathcal{P}' \underset{\scriptscriptstyle{{\mathsf{RS}(\mathcal{T})}}}{\leftrightharpoons} \mathsf{Adv} \right) + \max_{\mathsf{T}, \mathsf{T}' \in \mathbf{T}} D_{\alpha}( \mathsf{RS}(\mathsf{T}) || \mathsf{RS}(\mathsf{T}') ) \leq 
 \label{eqn:prop3_3}
  \end{equation} 
 
\begin{equation}
D_{\alpha}\left(\mathcal{P} \underset{\scriptscriptstyle{{\mathsf{RS}(\mathcal{T})}}}{\leftrightharpoons} \mathsf{Adv}||\mathcal{P}' \underset{\scriptscriptstyle{{\mathsf{RS}(\mathcal{T})}}}{\leftrightharpoons} \mathsf{Adv} \right) +  \max_{\mathsf{T}, \mathsf{T}' \in \mathbf{T}} D_{\infty}( \mathsf{RS}(\mathsf{T}) || \mathsf{RS}(\mathsf{T}') ) \leq 
 \label{eqn:prop3_4}
\end{equation}

\begin{equation}
 \epsilon +  \log(BC) 
  \label{eqn:prop3_5}
\end{equation}
 
 Eqn. \ref{eqn:prop3_2} follows from the application of the identity. The symbol $\leftrightharpoons$ denotes a shorthand for the realized query-response sequence.

 Eqn. \ref{eqn:prop3_3} follows from upper bounding the maximal divergence conditioned over all sequences $\leftrightharpoons$ with the maximal divergence over all possible termination rules. This is immediate given that any the termination rule conditioned on a sequence is contained in $\mathbf{T}$.

Eqn. \ref{eqn:prop3_4} follows from $D_\alpha \leq D_\infty$ for any $\alpha$. The first term of Eqn. \ref{eqn:prop3_5} follows from the $(\alpha,\epsilon)$-ROP assumption on $(\mathcal{P},\mathcal{T})$, and the second term follows from noting that $\mathbf{Pr}[\mathsf{RS}(\mathsf{T}) = t] \geq 1/BC$ for all $ t \in \{1,...B\}$.
 
\emph{Part (2)} By assumption, we have that:

$$ p =  \mathbb{E}_{\mathbf{x} \sim \mathcal{D}_{\text{test}}}[\mathbf{PP}_\mathcal{P}]  $$

before termination. We also have that 

$$ p_\emptyset =  \mathbb{E}_{\mathbf{x} \sim \mathcal{D}_{\text{test}}}[\mathbf{PP}_{h_\emptyset}]  $$

Let $V_t = \mathbf{1}\{ t < \mathsf{RS}(\mathcal{T}) \}$. Due to the stationarity, the expected perplexity is:

$$ \frac{1}{B}\sum_{t=1}^B pV_t + p_{\emptyset}\bar{V_t} $$

Computing the sum above yields the result. 
\end{proof}

We walk through an example in order to make the Prop. above more concrete. In Fig. 2, we report \textsc{SubMix} achieves $\mathbf{PP} = 26.9$ at $\epsilon = 2$ for $B=1000$ queries. Selecting a value of $C \gets 10$ results in $\mathbf{PP}_{C \gets 10} \leq 27.78$ at $\epsilon_{C \gets 10} = 11.21$. Selecting $C \gets 100$ results in $\mathbf{PP}_{C \gets 100} \leq 26.988$ at $\epsilon_{C \gets 100} = 13.51$. Selecting $C \gets 1$ results in $\mathbf{PP}_{C \gets 1} \leq 31.3$ at $\epsilon_{C \gets 1} = 8.9$.

\section{Text Extraction and Eidetic Memorization}

We formalize the notion of extraction through the framework of statistcal hypothesis testing.

\begin{definition} (Text Extraction Game)
An unknown target string $S \in \Sigma^*$ is a substring appearing in the corpus $\mathcal{D}$. Suppose $S$ has been narrowed down to one of $m$ values: $S \in \{ \mathfrak{s}_1,..., \mathfrak{s}_m\}$. An adversary $\mathsf{Adv}$ observes (or interacts with) mechanism (or protocol) $\mathcal{M}$ and outputs a guess $\mathsf{s} = \mathsf{Adv}(\mathcal{M})$.

\end{definition}

\begin{definition}($(\beta,m)$-Extractibility)
An unknown string $S$ is $(\beta,m)$-extractable from mechanism (or protocol) $\mathcal{M}$ if for any choice of $\mathcal{D}$ and  $(\mathfrak{s}_1,...,\mathfrak{s}_m)$, there exists an adversary $\mathsf{Adv}$ such that:
$$\min_{i}\mathbf{Pr}(S = \mathsf{s} | S = \mathfrak{s}_i) > \beta$$
\end{definition}

\begin{remark}
The reader may wonder why we use $\min_i$ rather than $\max_i$. The reason is because the attacker should be unbiased about which string they want to extract, so to maximize their success rate they should equalize this success rate across all strings $\mathfrak{s}_i$. 
\end{remark}

\begin{definition}($(\kappa,\beta,m)$-Eidetic Memorization)
A string $S$ appearing in at most $\kappa$ examples in the training
data $\mathcal{D}$ is $(\kappa,\beta,m)$-eidetic memorized if $S$ is $(\beta,m)$-extractable.
\end{definition}

\begin{theorem}
If a mechanism $\mathcal{M}$ is $(\alpha, \epsilon)$-RDP then $\mathcal{M}$ cannot $(\kappa ,\frac{\epsilon\kappa + \log(2)}{\log(m)},m)$-eidetically memorize any string $S$.
\end{theorem}

\begin{remark}
This result does not depend on $\alpha$ and therefore incurs more slack when $\alpha$ is large.
\end{remark}

\begin{proof}
To begin, note that $\kappa$-eidetic memorization limits the occurrences of $S$ in $\mathcal{D}$. Given that $S$ appears at most $\kappa$ times in $\mathcal{D}$, it follows that $S$ appears in at most $\kappa$ user texts:

 \begin{equation}
 \label{eqn:user_bound}
     |\{ \mathcal{D}_i : S \in \mathcal{D}_i\}| \leq \kappa
 \end{equation} 

Let $\mathcal{D}|\{S = \mathfrak{s}_i\}$ denote the dataset when $S$ takes on value $\mathfrak{s}_i$. 

Let $\mathcal{M}|\{S = \mathfrak{s}_i\}$ denote the output of the mechanism when $S$ takes on value $\mathfrak{s}_i$. 

Let $d_\mathcal{H}$ denote the user-level Hamming distance between datasets.

Based on (\ref{eqn:user_bound}) we know that for all $i,j \in [m]$:

\begin{equation}
    d_{\mathcal{H}}\left(\mathcal{D}|\{S = \mathfrak{s}_i\}, \mathcal{D}|\{S = \mathfrak{s}_j\}\right) \leq \kappa
\end{equation}

By assumption, $\mathcal{M}$ is $(\alpha, \epsilon)$-RDP, so from the above follows:
\begin{equation}
    D_{\alpha}\left(\mathcal{M}|\{S = \mathfrak{s}_i\}|| \mathcal{M}|\{S = \mathfrak{s}_j\}\right) \leq \kappa \epsilon
\end{equation}

By the monotonicity of Renyi divergence in $\alpha$:

\begin{equation}
     D_{\text{KL}}\left(\mathcal{M}|\{S = \mathfrak{s}_i\} || \mathcal{M}|\{S = \mathfrak{s}_j\}\right) \leq D_{\alpha}\left(\mathcal{M}|\{S = \mathfrak{s}_i\} || \mathcal{M}|\{S = \mathfrak{s}_j\}\right) \leq \kappa \epsilon
\end{equation}

With an upper bound on the KL-Divergence, the final step will be the application of Fano's inequality to multiple hypothesis testing \citep{rigollet201518}:

\begin{equation}
    \max_i\mathbf{Pr}(\mathsf{s} \neq S|S = \mathfrak{s}_i) \geq 1 - \frac{\kappa\epsilon + \log 2}{\log m}
\end{equation}
\end{proof}


\end{document}